\documentclass{article}

\usepackage{PRIMEarxiv}
\usepackage{booktabs}
\usepackage{graphicx}
\usepackage{amssymb}
\usepackage{longtable}
\usepackage{mathtools}
\usepackage{pdflscape}
\usepackage{amsthm}

\newtheorem{theorem}{Theorem}[section]

\newtheorem{proposition}[theorem]{Proposition}
\newtheorem{definition}{Definition}

\usepackage{algorithm}
\usepackage{algpseudocode}
\usepackage{natbib}
\makeatletter
\renewcommand\@biblabel[1]{}
\makeatother

\newcommand{\Ent}{\mathsf{H}}

\title{Loss-Shift Transfer via Bayes Quotients
}

\author{
  Vasileios Sevetlidis\\
  Athena Research Center, Kimmeria Campus, Xanthi, Greece \\
  Democritus University of Thrace, Vas. Sofias Campus, Xanthi, Greece \\
  International Hellenic University, Serres, Greece\\
  \texttt{vasiseve@athenarc.gr} \\
}

\begin{document}
\maketitle

\begin{abstract}
Transfer learning is usually studied as a consequence of distribution shift. This paper identifies an orthogonal failure mode in which the data distribution is fixed and the loss changes. This setting is called \emph{loss shift}. A loss determines which information in \(X\) is Bayes-relevant, and two losses may therefore require different representations even under the same joint law \(P(X,Y)\). The idea is formalized using Bayes quotients, which allow losses to be ordered by refinement. In the Bayes-quotient formulation, strict refinement gives an immediate qualitative obstruction. A source-minimal representation for a coarser loss is insufficient for a strictly finer target loss. For finite-output log loss, this obstruction becomes an exact quantitative identity. The excess risk is the conditional information about \(Y\) discarded by the representation. Experiments in controlled, learned, synthetic-image, and real-image settings show the predicted effect, i.e., classification-equivalent representations can have different optimal log-loss performance under a fixed data distribution.
\end{abstract}


\section{Introduction}
\label{sec:introduction}

Transfer learning is usually studied as learning under change, i.e., a model trained in a source setting is reused in a target setting whose distribution, domain, or task may differ \citep{pan2010survey,zhuang2021comprehensive}. Existing theory therefore focuses on domain discrepancy, task relatedness, invariance, and shared structure \citep{ben2007analysis,ben2010theory,quionero2009dataset, ganin2016domain,bengio2013representation}. This paper studies an orthogonal source of transfer failure. The joint law \(P(X,Y)\) is fixed, and the loss changes. A representation trained for a coarse objective, such as classification accuracy, may be reused for a finer objective, such as probabilistic prediction under log loss. Under the fixed distribution, the target loss may require distinctions beyond those required by the source loss. This setting is called \emph{loss shift}.

The key point is decision-theoretic. A loss determines which aspect of \(P(Y\mid X)\) is Bayes-relevant \citep{wald1950statistical,berger1985statistical,devroye1996probabilistic}. Zero-one loss depends on the Bayes class; log loss requires the full predictive law \citep{gneiting2007strictly,gneiting2011making}. Thus, in binary classification, an accuracy-minimal representation may identify points with \(\eta(x)=P(Y=1\mid X=x)=0.51\) and \(\eta(x)=0.99\), since both have the same Bayes class. Under log loss, these points require different forecasts. Once a frozen representation has collapsed this distinction, the distinction is unavailable to downstream heads trained on that representation.

This is formalized using \emph{Bayes quotients} induced by the loss. For fixed \(P\) and loss \(\ell\), the Bayes quotient identifies inputs that are equivalent for Bayes-optimal prediction under \(\ell\). A representation is Bayes-sufficient for \((P,\ell)\) exactly when it preserves this quotient. This construction allows losses to be compared by refinement. The notation \(\ell_1\preceq_P\ell_2\) means that every distinction required by \(\ell_1\) is also required by \(\ell_2\), and \(\ell_1\prec_P\ell_2\) denotes strict refinement. The quotient formulation makes the qualitative obstruction transparent. If \(\ell_1\prec_P\ell_2\), then a representation that is Bayes-minimal for \(\ell_1\) is insufficient for \(\ell_2\). The main quantitative result specializes this obstruction to finite-output log loss, where the frozen-transfer gap is exactly
\[
R^\star_{\log,P}(h(X))-R^\star_{\log,P}(X)
=
I(Y;X\mid h(X)),
\]
equivalently
\[
\mathbb E\!\left[
D_{\mathrm{KL}}\bigl(P(Y\mid X)\,\|\,P(Y\mid h(X))\bigr)
\right].
\]
Thus the transfer penalty is precisely the conditional information about \(Y\) lost by the frozen representation \citep{cover2006elements}.

The contributions are as follows. \textbf{(i)} Loss shift is identified as a transfer failure mode distinct from distribution shift. \textbf{(ii)} Building on the Bayes-sufficient quotient framework of \citet{sevetlidis2026bayessufficient}, loss-induced Bayes quotients are used as the representation-theoretic substrate for comparing source and target losses. \textbf{(iii)} Losses are ordered by quotient refinement, and strict refinement is shown to give a qualitative obstruction for source-minimal frozen representations. \textbf{(iv)} For finite-output log loss, the exact frozen-transfer gap is identified as conditional mutual information, with accuracy-to-probabilistic-prediction transfer as the motivating example.

\section{Background and Related Work}
\label{sec:background}

This work connects transfer learning, decision theory, property elicitation, representation sufficiency, and calibration. Its central premise is that the information a representation must preserve is determined by the pair \((P,\ell)\).

Transfer learning is commonly formulated as reuse across source and target domains or tasks \citep{pan2010survey,zhuang2021comprehensive}. In domain adaptation, target risk is controlled by source risk, a source--target discrepancy, and the error of an ideal joint hypothesis \citep{ben2007analysis,ben2010theory}. Related work on dataset shift studies covariate shift, sample-selection bias, label shift, and concept shift as mismatches between training and deployment distributions \citep{quionero2009dataset,sugiyama2012machine}. Deep transfer and domain-adversarial methods similarly seek representations that align domains while preserving predictive information \citep{ganin2016domain}. Later work shows that invariance alone can discard class-relevant structure or fail to preserve the conditionals needed for prediction \citep{zhao2019learning,stojanov2021domain,zhao2022fundamental}.

The setting here is orthogonal. The joint law \(P(X,Y)\) is fixed and the loss changes. Transfer failure is driven by a mismatch between the information required by the source and target objectives.

Decision theory views prediction as choosing actions under a loss \citep{wald1950statistical,berger1985statistical}. Given \(P(Y\mid X=x)\), the Bayes action minimizes conditional risk. Thus a loss specifies which property of the conditional law is relevant, e.g., the Bayes class for zero-one loss, the conditional mean for squared loss, a median or quantile for absolute and quantile losses, and the full conditional distribution for log loss. Property elicitation formalizes this perspective by studying which statistical properties are elicited by which losses \citep{savage1971elicitation,lambert2008eliciting,gneiting2011making, frongillo2015elicitation} and are exactly the candidates for the Bayes quotients used below. Proper scoring rules are the canonical losses for probabilistic prediction; strictly proper rules are uniquely minimized by the true predictive distribution \citep{gneiting2007strictly,dawid2007geometry, dawid2014theory}.

This viewpoint is used at the representation level. For fixed \(P\), a loss induces a Bayes-relevant partition, or quotient, of the input space. Different losses may induce different quotients, and a target loss is more demanding when it separates inputs that the source loss identifies.

The present paper builds directly on the Bayes-sufficient representation framework of \citet{sevetlidis2026bayessufficient}. That framework proves that, when the Bayes action is almost surely unique, Bayes sufficiency is equivalent to measurability of the Bayes action with respect to the representation, or equivalently to containment of the induced Bayes quotient. Here it is used as the representation-theoretic means for comparing losses. Different losses may induce different Bayes quotients under the same joint law \(P\), and strict refinement of these quotients yields an obstruction to frozen transfer.

The question of what a representation must preserve appears in sufficient dimension reduction, information bottleneck, and representation learning. Sufficient dimension reduction seeks \(Z=h(X)\) such that \(Y\perp X\mid Z\), thereby preserving the full conditional law \citep{li1991sliced,cook1998regression,chiaromonte2002sufficient, cook2009sufficient}. This is appropriate for log loss and can be stronger than necessary for coarser losses. Information bottleneck and modern representation learning emphasize compression while retaining task-relevant information \citep{tishby1999information,bengio2013representation, achille2018emergence}. The point here is that task relevance is loss-dependent. Information irrelevant for one loss can be essential for another.

This makes minimality central. A representation that is Bayes-minimal for a source loss preserves exactly the distinctions needed for that loss. Such a representation is optimal for the source objective and may be too compressed for a finer target objective. This perspective is related to Blackwell's comparison of experiments, which orders information structures by usefulness across decision problems \citep{blackwell1953equivalent,torgersen1991comparison}. Here the information source and joint law are fixed, and losses are ordered by the Bayes-relevant information they require from representations of \(X\).

Classification-to-probability-estimation is the clearest example of loss shift. Accuracy asks for a top-label decision; log-loss probability prediction asks for the relevant conditional probabilities. In binary classification, zero-one loss depends on the sign of \(\eta(X)-1/2\), and log loss depends on the value of \(\eta(X)=P(Y=1\mid X)\). An accuracy-minimal representation may therefore collapse inputs on the same side of the decision boundary even when they require different probability forecasts.

This distinction is related to work on class-probability estimation, surrogate and composite losses, and calibration \citep{niculescu2005predicting,buja2005loss,bartlett2006convexity, reid2010composite,guo2017calibration}. Modern neural networks may achieve high accuracy together with poor probability estimates, motivating post-hoc methods such as temperature scaling \citep{guo2017calibration}. From this perspective, calibration can also be an information problem. If the frozen representation has discarded the conditional probabilities evaluated by the target loss, those probabilities are unavailable to the downstream head.

\section{Theoretical Framework}
\label{sec:framework}

This section formalizes frozen-representation transfer under loss shift. The joint law \(P(X,Y)\) is fixed throughout, and the loss changes. The central object is the Bayes-relevant information induced by a loss under this fixed law.

The formalism is decision-theoretic. A loss determines a Bayes decision problem, and hence determines which features of the conditional law of \(Y\) given \(X\) are relevant for optimal prediction \citep{wald1950statistical,berger1985statistical,devroye1996probabilistic}. This viewpoint is also closely related to the theory of elicitable functionals. Different losses are consistent for different properties of the conditional distribution, such as the mean, a quantile, a class decision, or the full predictive law \citep{savage1971elicitation,gneiting2011making,gneiting2007strictly, lambert2008eliciting,frongillo2015vector}.

\subsection{Measurable setup and Bayes sufficiency}
\label{subsec:measurable-setup}

Let \((\mathcal X,\mathcal A_{\mathcal X})\) and \((\mathcal Y,\mathcal A_{\mathcal Y})\) be standard Borel spaces, and let
\begin{equation}
P\in\mathcal P(\mathcal X\times\mathcal Y),
\qquad
(X,Y)\sim P.
\end{equation}
The marginal law of \(X\) is denoted by \(P_X\). The standard Borel assumption is used to avoid measurability pathologies and to ensure the existence of regular conditional laws in the usual supervised settings \citep{kallenberg2002foundations,bogachev2007measure}.

Let \((\mathcal A,\mathcal A_{\mathcal A})\) be a measurable action space and let
\begin{equation}
\ell:\mathcal A\times\mathcal Y\to[0,\infty]
\end{equation}
be a measurable loss. A predictor is a measurable map \(f:\mathcal X\to\mathcal A\), with risk
\begin{equation}
R_{\ell,P}(f)
:=
\mathbb E_P[\ell(f(X),Y)].
\end{equation}
The unrestricted Bayes risk is
\begin{equation}
R^\star_{\ell,P}
:=
\inf_{f:\mathcal X\to\mathcal A}R_{\ell,P}(f),
\end{equation}
where the infimum is over all measurable predictors.

A representation is a measurable map \(h:\mathcal X\to\mathcal H\), with \(H:=h(X)\). The optimal risk achievable through \(H\) is
\begin{equation}
R^\star_{\ell,P}(H)
:=
\inf_{c:\mathcal H\to\mathcal A}
\mathbb E_P[\ell(c(H),Y)],
\end{equation}
where the infimum is over measurable heads \(c\). The representation-level excess risk is
\begin{equation}
\mathcal E_{\ell,P}(H)
:=
R^\star_{\ell,P}(H)-R^\star_{\ell,P},
\end{equation}
whenever the subtraction is well-defined. Since every predictor of the form \(c(H)\) is also a measurable function of \(X\),
\begin{equation}
R^\star_{\ell,P}(H)\geq R^\star_{\ell,P}.
\end{equation}

\begin{definition}[Bayes sufficiency]
A representation \(H=h(X)\) is Bayes-sufficient for \((P,\ell)\) if
\begin{equation}
R^\star_{\ell,P}(H)=R^\star_{\ell,P}.
\end{equation}
Equivalently, using \(H\) preserves the optimal achievable risk available from the full input \(X\).
\end{definition}

When Bayes actions are attained, this agrees with the usual action-based statement that some measurable head on \(H\) implements a Bayes-optimal action rule. The risk formulation is used because it is the object needed for the transfer comparisons below.

This notion is loss-dependent. For log loss, Bayes sufficiency coincides with preserving the conditional law \(P(Y\mid X)\). This connects it to sufficient dimension reduction, where one seeks reductions of \(X\) that preserve \(Y\mid X\), or sometimes only a lower-dimensional functional such as \(\mathbb E[Y\mid X]\) \citep{li1991sliced,cook1998regression,cook2009sufficient}. For other losses, Bayes sufficiency may require less information than the full conditional law.

All inclusions and equalities of sigma-algebras are understood modulo null sets. Since all representations are functions of \(X\), an inclusion such as
\begin{equation}
\sigma(q(X))\subseteq\sigma(h(X))
\end{equation}
may equivalently be read as the pullback inclusion
\begin{equation}
q^{-1}(\mathcal A_{\mathcal Q})
\subseteq
h^{-1}(\mathcal A_{\mathcal H})
\qquad
\text{mod }P_X.
\end{equation}

\subsection{Bayes-quotient setting}
\label{subsec:bayes-quotient-setting}

For a regular conditional law \(P(Y\in\cdot\mid X=x)\), define the conditional risk
\begin{equation}
L_{\ell,P}(a\mid x)
:=
\mathbb E[\ell(a,Y)\mid X=x].
\end{equation}
The Bayes action correspondence is
\begin{equation}
\Gamma_{\ell,P}(x)
:=
\arg\min_{a\in\mathcal A} L_{\ell,P}(a\mid x),
\end{equation}
whenever the minimum is attained. This is the usual conditional form of the Bayes decision problem \citep{wald1950statistical,berger1985statistical}.

\paragraph{Bayes-quotient regime.}
The Bayes-quotient framework of \citep{sevetlidis2026bayessufficient} is used. For a fixed joint law \(P\) and loss \(\ell\), a representation \(H=h(X)\) is Bayes-sufficient if it contains enough information to implement a Bayes-optimal decision rule. In the almost-surely unique Bayes-action case, let \(a^\star_{\ell,P}(X)\) denote the unique Bayes action. The Bayes information is then
\begin{equation}
I_{\ell,P}:=\sigma(a^\star_{\ell,P}(X)),
\end{equation}
and a Bayes quotient is any random variable
\begin{equation}
Q_{\ell,P}=q_{\ell,P}(X)
\end{equation}
such that
\begin{equation}
\sigma(Q_{\ell,P})=I_{\ell,P}
\qquad
\mathrm{mod}\;P_X .
\end{equation}
The quotient characterization of Bayes-sufficient representations gives
\begin{equation}
\begin{aligned}
&H \text{ is Bayes-sufficient for } (P,\ell)\\
&\quad\Longleftrightarrow\quad
\sigma(q_{\ell,P}(X))\subseteq\sigma(H)
\qquad \mathrm{mod}\;P_X .
\end{aligned}
\end{equation}
Thus the Bayes quotient is the loss-dependent representation target obtained from the Bayes-sufficiency factorization theorem. The present paper takes this quotient-representable setting and studies what happens when two losses induce different quotients under the same joint law \(P\).

The standard losses considered below lie in this regime under the stated regularity conditions. For example, zero-one classification without ties has a unique Bayes class, squared loss has the conditional mean as Bayes action, unique quantile losses have the corresponding conditional quantile as Bayes action, and finite-label log loss has the conditional class-probability vector as Bayes action. More generally, the same notation applies whenever the Bayes-relevant information for \((P,\ell)\) is represented by a measurable quotient \(q_{\ell,P}(X)\) satisfying the display above.

\begin{definition}[Bayes minimality]
In the Bayes-quotient regime, a representation \(H=h(X)\) is Bayes-minimal for \((P,\ell)\) if
\begin{equation}
\sigma(H)=\sigma(q_{\ell,P}(X))
\qquad
\mathrm{mod}\;P_X .
\end{equation}
\end{definition}

Thus Bayes sufficiency means that \(H\) contains the Bayes quotient, and Bayes minimality means that \(H\) contains exactly the Bayes quotient.

\subsection{Quotients for standard losses}
\label{subsec:standard-loss-quotients}

The Bayes quotient is concrete in the standard losses used in supervised learning. When the Bayes action is almost surely unique, the quotient is generated by the Bayes action \(a^\star_{\ell,P}(X)\). For zero-one classification without ties, this is the Bayes class. In binary classification with \(\eta(X)=P(Y=1\mid X)\), the quotient is generated by
\begin{equation}
\mathbf 1\{\eta(X)>1/2\}.
\end{equation}
For squared loss with integrable real-valued \(Y\), the quotient is generated by the conditional mean \(\mathbb E[Y\mid X]\). For losses eliciting a unique quantile or median, the quotient is generated by that conditional functional.

Finite-label log loss is finer. If \(\mathcal Y=\{1,\dots,K\}\), then log loss is minimized by the full conditional probability vector
\begin{equation}
\pi_X=(P(Y=1\mid X),\dots,P(Y=K\mid X)).
\end{equation}
Thus the log-loss Bayes quotient is generated by \(\pi_X\), as in classical characterizations of proper scoring rules and conditional risk minimization \citep{gneiting2007strictly,steinwart2008support,shalev2014understanding}.

These examples make the refinement relation below explicit. Accuracy can identify inputs with the same Bayes class, and log loss separates inputs whenever their conditional probability vectors differ.

\subsection{Loss preorder}
\label{subsec:loss-preorder}

Let \(\ell_1:\mathcal A_1\times\mathcal Y\to[0,\infty]\) and \(\ell_2:\mathcal A_2\times\mathcal Y\to[0,\infty]\) be two losses that admit Bayes quotients under \(P\), with quotient maps \(q_{\ell_1,P}\) and \(q_{\ell_2,P}\).

\begin{definition}[Loss preorder]
The loss \(\ell_2\) is said to be at least as representationally demanding as \(\ell_1\) under \(P\), written
\begin{equation}
\ell_1\preceq_P\ell_2,
\end{equation}
if
\begin{equation}
\sigma(q_{\ell_1,P}(X))
\subseteq
\sigma(q_{\ell_2,P}(X))
\qquad
\text{mod }P_X.
\end{equation}
Equivalently, every representation sufficient for \(\ell_2\) is sufficient for \(\ell_1\).
\end{definition}

\begin{definition}[Strict loss refinement]
The notation \(\ell_1\prec_P\ell_2\) means that \(\ell_1\preceq_P\ell_2\) and
\begin{equation}
\sigma(q_{\ell_2,P}(X))
\not\subseteq
\sigma(q_{\ell_1,P}(X))
\qquad
\text{mod }P_X.
\end{equation}
\end{definition}

\begin{definition}[Loss equivalence]
The notation \(\ell_1\sim_P\ell_2\) means that
\begin{equation}
\sigma(q_{\ell_1,P}(X))
=
\sigma(q_{\ell_2,P}(X))
\qquad
\text{mod }P_X.
\end{equation}
Thus \(\ell_1\) and \(\ell_2\) elicit the same Bayes-relevant information under \(P\).
\end{definition}

The relation \(\preceq_P\) orders losses by the sigma-algebras of Bayes-relevant information they induce. It is therefore close in spirit to comparing elicited properties of a distribution \citep{lambert2008eliciting,frongillo2015vector}, and retains the decision-theoretic flavor of Blackwell comparison \citep{blackwell1953equivalent,torgersen1991comparison}. In the present setting, the information source \(X\) is fixed and the ordering is over losses.

\begin{proposition}[The loss relation is a preorder]
\label{prop:loss-preorder}
The relation \(\preceq_P\) is reflexive and transitive. Moreover, it induces a partial order on equivalence classes modulo \(\sim_P\).
\end{proposition}

\begin{proof}[Proof sketch]
Reflexivity follows because every sigma-algebra contains itself. Transitivity follows from transitivity of sigma-algebra inclusion. After quotienting by \(\sim_P\), antisymmetry holds by construction. The full derivation is given in Appendix~\ref{app:proof-loss-preorder}.
\end{proof}

\subsection{Frozen transfer and qualitative obstruction}
\label{subsec:frozen-transfer}
\label{subsec:loss-shift-insufficiency}

Let \(\ell_{\mathrm{pre}}\) be a pretraining loss and \(\ell_{\mathrm{tar}}\) a target loss, both admitting Bayes quotients under \(P\). The joint law \(P(X,Y)\) is the same during pretraining and downstream evaluation; only the loss changes.

A source representation is Bayes-minimal when it contains exactly the source-loss quotient,
\begin{equation}
\sigma(H)=\sigma(q_{\ell_{\mathrm{pre}},P}(X))
\qquad
\mathrm{mod}\;P_X .
\end{equation}
This is a representation-level selection condition, i.e., source risk fixes the Bayes predictor, not every extra coordinate a representation might retain. The transfer results below therefore concern frozen representations that are source-minimal, or deliberately compressed toward the source quotient, as can arise through architecture, bottlenecking, regularization, optimization bias, or explicit minimality constraints. This distinction between predictor-level and representation-level properties is formalized by the fiber criterion of \citet{sevetlidis2026fiber}.

Pretraining is modeled as producing a frozen representation
\begin{equation}
H_{\mathrm{pre}}=h_{\mathrm{pre}}(X).
\end{equation}
The downstream learner may choose a measurable head \(c:\mathcal H_{\mathrm{pre}}\to\mathcal A_{\mathrm{tar}}\) to minimize
\begin{equation}
\mathbb E_P\!\left[\ell_{\mathrm{tar}}(c(H_{\mathrm{pre}}),Y)\right].
\end{equation}
The optimal frozen-transfer risk is
\begin{equation}
R^\star_{\ell_{\mathrm{tar}},P}(H_{\mathrm{pre}})
=
\inf_{c:\mathcal H_{\mathrm{pre}}\to\mathcal A_{\mathrm{tar}}}
\mathbb E_P[\ell_{\mathrm{tar}}(c(H_{\mathrm{pre}}),Y)],
\end{equation}
and the frozen-transfer excess risk is
\begin{equation}
\mathcal E_{\ell_{\mathrm{tar}},P}(H_{\mathrm{pre}})
=
R^\star_{\ell_{\mathrm{tar}},P}(H_{\mathrm{pre}})
-
R^\star_{\ell_{\mathrm{tar}},P}.
\end{equation}
The goal is to identify when this quantity is positive under a fixed distribution.

This gives a fixed-distribution counterpart to the usual transfer-learning setup, which primarily studies changes between source and target distributions or domains \citep{pan2010survey,zhuang2021comprehensive}. Here the predictive criterion changes under the same joint law.

Once losses are represented by their Bayes quotients, strict refinement has an immediate transfer consequence. A representation that is minimal for a coarser source quotient is insufficient for a strictly finer target quotient.

\begin{proposition}[Qualitative loss-shift obstruction]
\label{prop:loss-shift-obstruction}
Let \(\ell_1\) and \(\ell_2\) be losses that admit Bayes quotients under \(P\). Assume
\begin{equation}
\ell_1\prec_P\ell_2.
\end{equation}
If \(H_1^\star=h_1^\star(X)\) is Bayes-minimal for \((P,\ell_1)\), then \(H_1^\star\) is not Bayes-sufficient for \((P,\ell_2)\).
\end{proposition}

\begin{proof}[Proof sketch]
Bayes minimality gives
\begin{equation}
\sigma(H_1^\star)=\sigma(q_{\ell_1,P}(X))
\qquad
\mathrm{mod}\;P_X .
\end{equation}
Strict refinement gives
\begin{equation}
\sigma(q_{\ell_2,P}(X))
\not\subseteq
\sigma(q_{\ell_1,P}(X))
\qquad
\mathrm{mod}\;P_X .
\end{equation}
Therefore \(\sigma(q_{\ell_2,P}(X))\not\subseteq\sigma(H_1^\star)\). By the Bayes-quotient characterization recalled above, \(H_1^\star\) is not Bayes-sufficient for \((P,\ell_2)\). The full derivation is given in Appendix~\ref{app:proof-loss-shift-obstruction}.
\end{proof}

Consequently, when the target Bayes risk is finite,
\begin{equation}
\mathcal E_{\ell_2,P}(H_1^\star)
=
R^\star_{\ell_2,P}(H_1^\star)
-
R^\star_{\ell_2,P}
\in (0,\infty],
\end{equation}
with finite positive excess whenever \(R^\star_{\ell_2,P}(H_1^\star)<\infty\). Equivalently, optimizing over all measurable downstream heads cannot recover the unrestricted target Bayes risk.

Because the obstruction is stated for all measurable heads, it is stronger than a linear-probe obstruction. Linear probes are commonly used to evaluate frozen representations in representation learning \citep{alain2017understanding,he2020momentum,chen2020simple}. If \(\mathcal C_{\mathrm{lin}}\) is any class of linear or affine heads satisfying
\begin{equation}
\mathcal C_{\mathrm{lin}}
\subseteq
\{c:\mathcal H_1\to\mathcal A_2 \mid c \text{ is measurable}\},
\end{equation}
then
\begin{equation}
\inf_{c\in\mathcal C_{\mathrm{lin}}}
\mathbb E_P[\ell_2(c(H_1^\star),Y)]
\geq
\inf_{c:\mathcal H_1\to\mathcal A_2}
\mathbb E_P[\ell_2(c(H_1^\star),Y)].
\end{equation}
Therefore any measurable-head lower bound
\begin{equation}
\inf_{c:\mathcal H_1\to\mathcal A_2}
\mathbb E_P[\ell_2(c(H_1^\star),Y)]
-
R^\star_{\ell_2,P}
\geq
\Delta
\end{equation}
implies the corresponding linear-probe lower bound
\begin{equation}
\inf_{c\in\mathcal C_{\mathrm{lin}}}
\mathbb E_P[\ell_2(c(H_1^\star),Y)]
-
R^\star_{\ell_2,P}
\geq
\Delta.
\end{equation}
Thus the loss-shift obstruction holds at the level of all measurable heads and automatically transfers to linear-probe evaluation.

\subsection{Exact log-loss transfer gap}
\label{subsec:exact-log-loss-gap}

The qualitative proposition identifies when frozen transfer must fail. For finite-label log loss, the failure has an exact quantitative form. The excess target risk is precisely the conditional information about \(Y\) lost by the representation.

Let
\begin{equation}
\mathcal Y=\{1,\dots,K\},
\end{equation}
let \(\Delta(\mathcal Y)\) be the probability simplex over \(\mathcal Y\), and define
\begin{equation}
\ell_{\log}(p,y):=-\log p(y),
\end{equation}
with \(-\log 0=+\infty\).

Let
\begin{equation}
\pi_X(y):=P(Y=y\mid X),
\qquad
\pi_H(y):=P(Y=y\mid H).
\end{equation}
Since \(H=h(X)\),
\begin{equation}
\pi_H(y)=\mathbb E[\pi_X(y)\mid H],
\qquad
y\in\mathcal Y.
\end{equation}
Thus \(D_{\mathrm{KL}}(\pi_X\|\pi_H)\) is well-defined almost surely. The convention is
\begin{equation}
D_{\mathrm{KL}}(p\|q)
:=
\sum_{y\in\mathcal Y}
p(y)\log\frac{p(y)}{q(y)},
\end{equation}
with \(0\log(0/q)=0\) and \(p\log(p/0)=+\infty\) for \(p>0\). The notation \(\Ent(\cdot\mid\cdot)\) denotes conditional Shannon entropy.

\begin{theorem}[Exact log-loss transfer gap]
\label{thm:exact-log-loss-gap}
For any measurable representation \(H=h(X)\),
\begin{equation}
\mathcal E_{\log,P}(H)
=
R^\star_{\log,P}(H)-R^\star_{\log,P}
=
\mathbb E\left[
D_{\mathrm{KL}}\bigl(\pi_X\|\pi_H\bigr)
\right].
\end{equation}
Equivalently,
\begin{equation}
\mathcal E_{\log,P}(H)
=
\Ent(Y\mid H)-\Ent(Y\mid X)
=
I(Y;X\mid H).
\end{equation}
\end{theorem}

\begin{proof}[Proof sketch]
For fixed \(H=u\), the log-risk of a head \(c(u)\) is the cross-entropy between \(\pi_H(u)\) and \(c(u)\). This is minimized at \(c(u)=\pi_H(u)\), giving \(R^\star_{\log,P}(H)=\Ent(Y\mid H)\). Taking \(H=X\) gives \(R^\star_{\log,P}=\Ent(Y\mid X)\). The entropy difference equals the expected conditional KL divergence, and because \(H\) is a function of \(X\), it also equals \(I(Y;X\mid H)\). The full derivation is given in Appendix~\ref{app:proof-exact-log-loss-gap}.
\end{proof}

This identity is the standard log-loss regret decomposition into conditional entropy, conditional KL divergence, and conditional mutual information \citep{cover2006elements}. Related excess-log-risk and conditional-information representations also appear in Bayesian learning theory \citep{xu2022minimum}.

Combining Theorem~\ref{thm:exact-log-loss-gap} with Proposition~\ref{prop:loss-shift-obstruction}, suppose that \(\ell_1\) is a pretraining loss admitting a Bayes quotient under \(P\) and that
\begin{equation}
\ell_1\prec_P\ell_{\log}.
\end{equation}
If \(H_1^\star=h_1^\star(X)\) is Bayes-minimal for \((P,\ell_1)\), then
\begin{equation}
\mathcal E_{\log,P}(H_1^\star)
=
\mathbb E\left[
D_{\mathrm{KL}}
\bigl(
P(Y\mid X)
\,\|\,
P(Y\mid H_1^\star)
\bigr)
\right]
>0.
\end{equation}

\subsection{Binary classification-to-log-loss transfer}
\label{subsec:binary-accuracy-calibration}

The log-loss result is now instantiated in binary classification. Let \(\mathcal Y=\{0,1\}\) and
\begin{equation}
\eta(X):=P(Y=1\mid X).
\end{equation}
The zero-one loss is
\begin{equation}
\ell_{0/1}(a,y):=\mathbf 1\{a\neq y\},
\qquad
a\in\{0,1\}.
\end{equation}
Assume there are no ties.
\begin{equation}
P_X(\eta(X)=1/2)=0.
\end{equation}
Then the unique zero-one Bayes action is
\begin{equation}
a^\star_{0/1,P}(X)
=
\mathbf 1\{\eta(X)>1/2\},
\end{equation}
and a zero-one Bayes-minimal representation is
\begin{equation}
H^\star_{0/1}
:=
\mathbf 1\{\eta(X)>1/2\}.
\end{equation}

Binary log loss elicits the full conditional probability \(\eta(X)\), extending beyond its thresholded decision. This reflects the classical distinction between classification and class-probability estimation. Zero-one loss requires the Bayes class decision, and proper losses for probability estimation require the conditional class probability \citep{buja2005loss,reid2010composite,bartlett2006convexity, niculescu2005predicting,guo2017calibration,resin2023classification}.

\begin{proposition}[Binary classification-to-log-loss gap]
\label{prop:binary-accuracy-calibration-gap}
Assume \(\eta(X)\) is not measurable with respect to \(\sigma(H^\star_{0/1})\). Then \(H^\star_{0/1}\) is not Bayes-sufficient for binary log loss, and
\begin{equation}
\mathcal E_{\log,P}(H^\star_{0/1})
=
\mathbb E\left[
\mathrm{kl}
\left(
\eta(X)
\,\middle\|\,
\mathbb E[\eta(X)\mid H^\star_{0/1}]
\right)
\right]
>0,
\end{equation}
where
\begin{equation}
\mathrm{kl}(u\|v)
:=
u\log\frac{u}{v}
+
(1-u)\log\frac{1-u}{1-v}.
\end{equation}
\end{proposition}

\begin{proof}[Proof sketch]
Binary log loss is sufficient exactly when the representation determines \(\eta(X)\). Since \(H^\star_{0/1}\) only records the Bayes class decision, it is insufficient whenever \(\eta(X)\) is not \(\sigma(H^\star_{0/1})\)-measurable. Applying Theorem~\ref{thm:exact-log-loss-gap} and using the tower property gives the displayed binary KL identity. Positivity follows from strict positivity of binary KL whenever its arguments differ. The full derivation is given in Appendix~\ref{app:proof-binary-accuracy-calibration-gap}.
\end{proof}

\section{Experiments}

\label{sec:experiments}

The theory predicts an obstruction to frozen transfer under identical source and target distributions. The experiments isolate the effect of changing the loss while keeping \(P(X,Y)\) fixed. They focus on the accuracy-to-log-loss setting, where the source objective requires the Bayes class and the target objective requires calibrated probabilities.

Four experimental settings are considered. The first is a fully controlled discrete model in which the Bayes quotient for classification is known exactly and the log-loss transfer gap can be computed in closed form. This experiment directly tests the conditional-KL identity in Theorem~\ref{thm:exact-log-loss-gap}. The second studies learned bottlenecked representations on the same distribution, showing that a classification objective together with bottleneck noise can induce the same loss-shift failure mode in a learned representation. The third uses an image-based construction in which the Bayes class is a coarse visual attribute, while the probability level depends on finer visual information. The fourth uses CIFAR-10H human label distributions, giving a real-image instance in which majority-label classification and soft-label prediction are evaluated on the same images.

In each experiment, source training and downstream evaluation use samples from the same joint law \(P(X,Y)\); the training/evaluation loss and the frozen representation constraint change. For each frozen representation, a downstream probabilistic head is trained on a training sample and evaluated on an independent test sample. The reported metrics are classification accuracy, negative log likelihood (NLL), Brier score, and expected calibration error (ECE). Since NLL is the proper scoring rule corresponding to the target loss, it is the primary metric. Brier score provides a secondary proper-loss evaluation. ECE is reported as a diagnostic. A coarsened predictor may be calibrated relative to its own representation while remaining strictly suboptimal under log loss.

\subsection{Controlled binary model with exact loss-shift gap}
\label{subsec:exp-controlled}

The theory is first tested in a controlled setting where the source and target distributions are identical, the Bayes quotient for classification is known exactly, and the log-loss transfer gap can be computed in closed form. This design isolates loss shift with domain, covariate, label, and concept held fixed.

Let
\[
X=(S,T),
\qquad
S\in\{-1,+1\},
\qquad
T\in\{1,\dots,m\}.
\]
The variables \(S\) and \(T\) are sampled independently and uniformly, and then
\[
Y\mid S,T \sim \operatorname{Bernoulli}(\eta(S,T)).
\]
The conditional probability is chosen so that
\[
\eta(+1,T)>1/2,
\qquad
\eta(-1,T)<1/2
\]
for every \(T\), while still varying with \(T\) inside each Bayes-class region. The main experiment uses \(m=5\) and
\[
\eta(+1,T)\in\{0.55,0.65,0.75,0.85,0.95\},
\]
with the symmetric values
\[
\eta(-1,T)\in\{0.45,0.35,0.25,0.15,0.05\}.
\]
Thus the Bayes classifier depends on \(S\), and the Bayes-optimal log-loss predictor depends on the full pair \((S,T)\).

The classification-minimal representation is
\[
H_{\mathrm{cls}}=S.
\]
It preserves exactly the Bayes class and is therefore sufficient for zero-one loss. It leaves \(\eta(S,T)\) undetermined and is therefore insufficient for log loss. The optimal log-loss head on \(H_{\mathrm{cls}}\) predicts the coarsened conditional probability
\[
P(Y=1\mid H_{\mathrm{cls}})
=
\mathbb E[\eta(S,T)\mid S].
\]
Consequently, the exact excess log risk of \(H_{\mathrm{cls}}\) relative to the full representation \(X\) is
\[
\mathbb E\left[
\mathrm{kl}\left(
\eta(S,T)
\,\middle\|\,
\mathbb E[\eta(S,T)\mid S]
\right)
\right],
\]
which is the binary specialization of Theorem~\ref{thm:exact-log-loss-gap}. For the main \(m=5\) setting, this population conditional-KL gap is \(0.058506\) nats for \(H_{\mathrm{cls}}=S\).

Five representations are compared. The full representation \(X=(S,T)\) preserves both the Bayes class and the probability level. The representation \(S\) preserves the Bayes class. A partially informative representation preserves \(S\) and a coarsening of \(T\), thereby retaining part of the within-class probability information. A hashed representation applies a lossy compression to \(X\). It is a fixed deterministic random hash from the ten cells \((S,T)\) into four codes, sampled once with seed \(17\), held fixed across replications, and allowed to merge different values of \(S\). Finally, a constant representation discards all input information. For each representation, the optimal downstream probabilistic head is estimated from \(20{,}000\) training samples and evaluated on \(50{,}000\) independent test samples. The experiment is repeated over \(300\) independent replications, and means with \(95\%\) confidence-interval half-widths are reported.

\begin{table*}[t]
\centering
\caption{
Controlled binary model. The full representation and the Bayes-class representation \(S\) achieve the same classification accuracy. The Bayes-class representation incurs substantially larger log loss and Brier score. The \(\pm\) values are \(95\%\) confidence-interval half-widths computed over \(300\) independent replications.
}
\label{tab:controlled-main}
\begin{tabular}{lcccc}
\toprule
Representation & Accuracy & NLL & Brier & ECE \\
\midrule
Full \(X=(S,T)\)
& \(0.7500 \pm 0.0003\)
& \(0.5040 \pm 0.0003\)
& \(0.1676 \pm 0.0001\)
& \(0.0082 \pm 0.0003\) \\
Bayes class \(S\)
& \(0.7500 \pm 0.0003\)
& \(0.5624 \pm 0.0003\)
& \(0.1875 \pm 0.0001\)
& \(0.0038 \pm 0.0003\) \\
Coarsened \(T\) within \(S\)
& \(0.7500 \pm 0.0003\)
& \(0.5229 \pm 0.0003\)
& \(0.1725 \pm 0.0002\)
& \(0.0054 \pm 0.0003\) \\
Hashed compression
& \(0.5401 \pm 0.0003\)
& \(0.6872 \pm 0.0001\)
& \(0.2470 \pm 0.0001\)
& \(0.0056 \pm 0.0003\) \\
Constant
& \(0.5001 \pm 0.0003\)
& \(0.6932 \pm 0.0000\)
& \(0.2500 \pm 0.0000\)
& \(0.0033 \pm 0.0003\) \\
\bottomrule
\end{tabular}
\end{table*}

Table~\ref{tab:controlled-main} shows the main result. The full representation and the Bayes-class representation \(S\) have indistinguishable classification accuracy,
\[
\operatorname{Acc}(X)=\operatorname{Acc}(S)=0.7500.
\]
Thus, from the perspective of the source objective, \(S\) is already optimal. Under the target log loss, the two representations are sharply separated. The full representation achieves NLL \(0.5040\), and the Bayes-class representation achieves NLL \(0.5624\). The resulting excess test log loss is
\[
0.0584 \pm 0.0002
\qquad
(95\%\text{ CI half-width}).
\]
This matches the exact population conditional-KL gap \(0.058506\) nats. The Brier score exhibits the same ordering. The representation \(S\) therefore preserves all information needed for Bayes-optimal classification and discards information needed for optimal probabilistic prediction.

\begin{figure}[t]
\centering
\includegraphics[width=.48\textwidth]{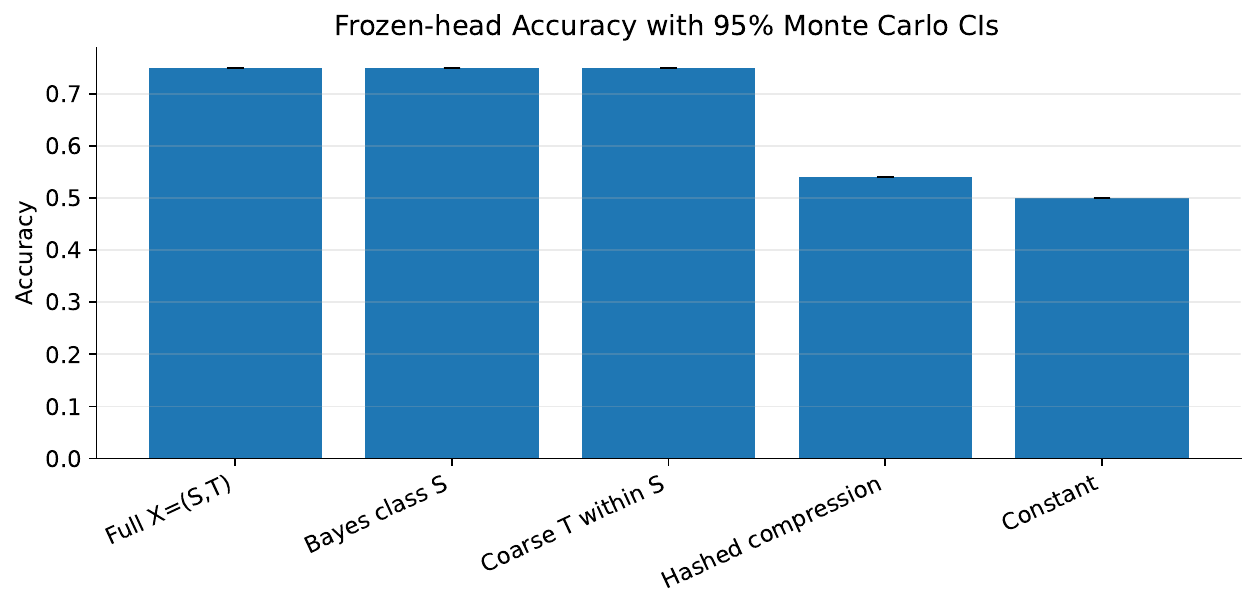}
\caption{
Classification accuracy in the controlled binary model. The full representation \(X=(S,T)\), the Bayes-class representation \(S\), and the partially coarsened representation achieve the same Bayes-optimal accuracy. Accuracy therefore treats these different amounts of within-class probability information as equivalent.
}
\label{fig:controlled-accuracy}
\end{figure}

\begin{figure}[t]
\centering
\includegraphics[width=.48\textwidth]{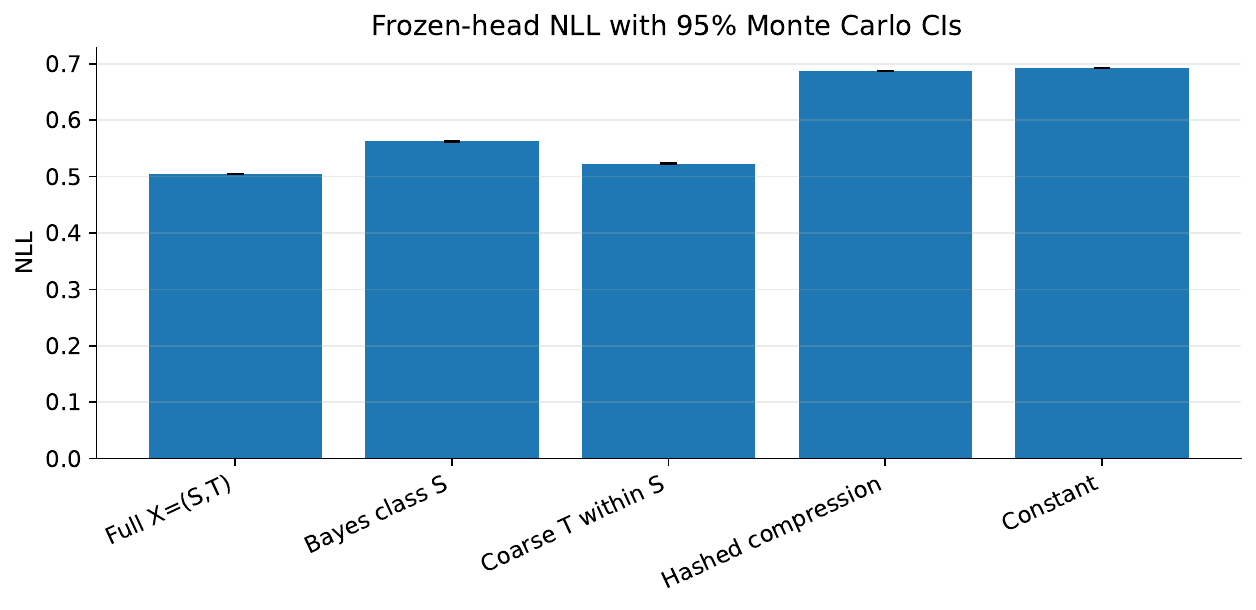}
\caption{
Negative log likelihood in the controlled binary model. The representation \(S\) is classification-optimal and strictly suboptimal for log loss because it collapses inputs with different conditional probabilities. The partially coarsened representation lies between the full and Bayes-class representations, consistent with retaining part of the log-loss-relevant information.
}
\label{fig:controlled-nll}
\end{figure}

Figures~\ref{fig:controlled-accuracy} and~\ref{fig:controlled-nll} visualize the loss-shift effect. Moving from \(X\) to \(S\) preserves accuracy and substantially increases NLL. This is the empirical signature predicted by Proposition~\ref{prop:loss-shift-obstruction}, i.e., a representation can be minimal and sufficient for the source loss while being insufficient for the target loss.

The intermediate representation provides a useful refinement check. It retains some information about \(T\) within each value of \(S\), and its NLL lies between the full representation and the Bayes-class representation. This monotone ordering supports the interpretation of the gap as an information-loss effect.

\begin{table}[t]
\centering
\caption{
Excess NLL relative to the full representation. The Bayes-class representation has a strictly positive log-loss gap despite having the same classification accuracy as the full representation. Entries are means \(\pm\) \(95\%\) confidence-interval half-widths.
}
\label{tab:controlled-gaps}
\begin{tabular}{lc}
\toprule
Representation & Excess NLL relative to \(X\) \\
\midrule
Bayes class \(S\)
& \(0.058389 \pm 0.000161\) \\
Coarsened \(T\) within \(S\)
& \(0.018914 \pm 0.000098\) \\
Hashed compression
& \(0.183139 \pm 0.000260\) \\
Constant
& \(0.189134 \pm 0.000262\) \\
\bottomrule
\end{tabular}
\end{table}

Table~\ref{tab:controlled-gaps} reports the NLL gaps relative to the full representation. The Bayes-class representation incurs a gap of \(0.058389\) nats, with a confidence interval far from zero. The partially coarsened representation has a smaller positive gap, and the hashed and constant representations are substantially worse. The ordering agrees with the amount of probability-relevant information preserved by each representation.

The amount of within-class probability variation is next varied while keeping the Bayes decision boundary fixed. Specifically, a one-parameter family indexed by \(\alpha\) is used, where increasing \(\alpha\) increases the variation of \(\eta(S,T)\) within each Bayes-class region. Throughout the sweep, the Bayes classifier remains \(S\). The growing log-loss transfer gap reflects the increasing importance of probability information discarded by \(S\).

\begin{figure}[t]
\centering
\includegraphics[width=.48\textwidth]{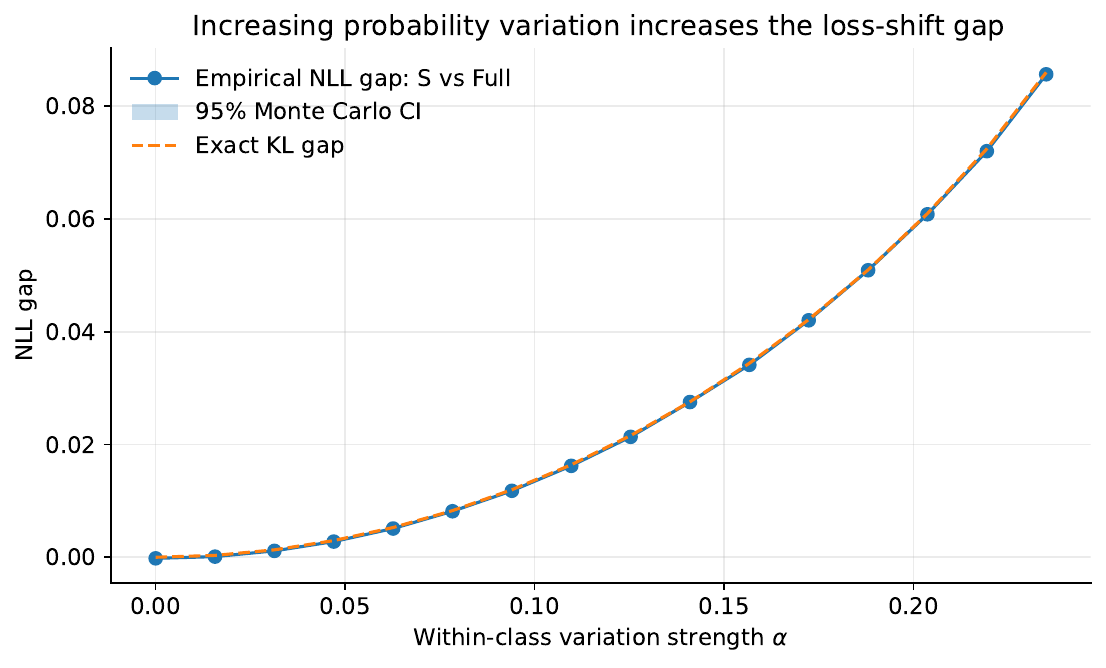}
\caption{
Excess log loss of the Bayes-class representation \(S\) as within-class probability variation increases. The Bayes classifier is unchanged throughout the sweep, and the log-loss penalty grows because \(S\) discards increasingly important probability information.
}
\label{fig:alpha-sweep-nll}
\end{figure}

\begin{figure}[t]
\centering
\includegraphics[width=.48\textwidth]{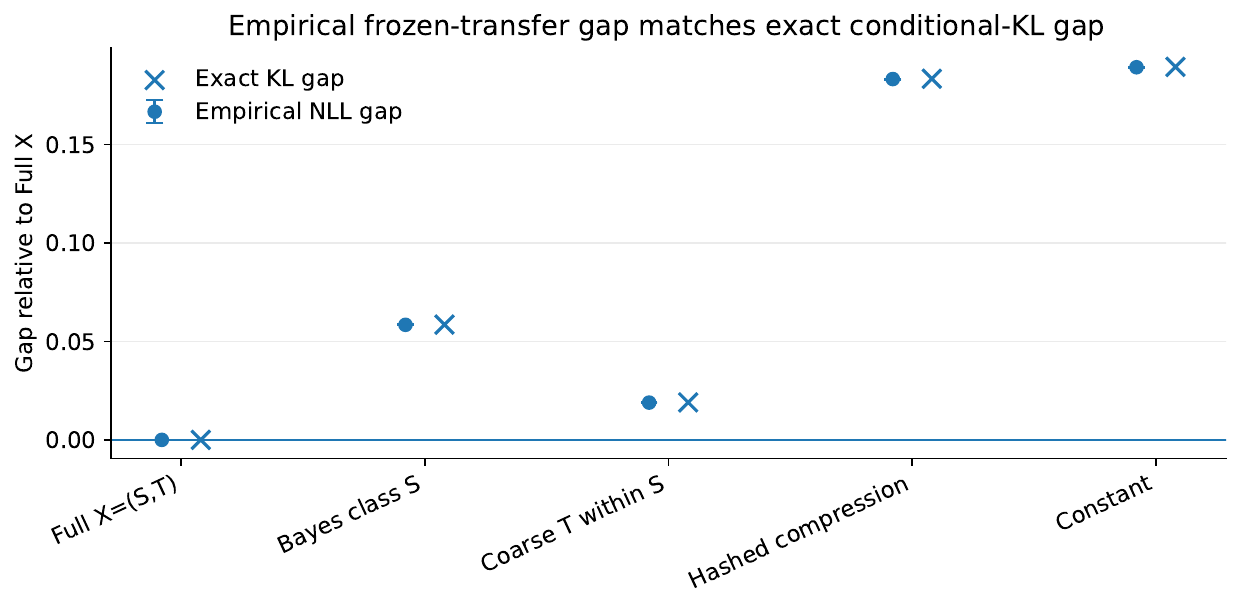}
\caption{
Empirical NLL gap versus the exact conditional-KL gap predicted by Theorem~\ref{thm:exact-log-loss-gap}. The empirical gap closely tracks the theoretical quantity, confirming that the observed transfer penalty is the conditional information about \(Y\) lost by the frozen representation.
}
\label{fig:empirical-vs-exact}
\end{figure}

Figures~\ref{fig:alpha-sweep-nll} and~\ref{fig:empirical-vs-exact} compare the empirical NLL gap with the exact conditional-KL gap. The agreement is nearly exact. Across the sweep, the correlation between the empirical and theoretical gaps is \(0.999995\), the mean absolute error is \(1.92\times 10^{-4}\), and the maximum absolute error is \(4.19\times 10^{-4}\). At the largest value of \(\alpha\), the exact gap is \(0.0860\) nats and the empirical gap is \(0.0857 \pm 0.0003\) using the confidence-interval half-width.

These results directly confirm the exact log-loss identity
\[
R^\star_{\log,P}(H)-R^\star_{\log,P}(X)
=
\mathbb E\left[
D_{\mathrm{KL}}\bigl(P(Y\mid X)\,\|\,P(Y\mid H)\bigr)
\right].
\]
They also illustrate why the obstruction is invisible to accuracy. Increasing within-class probability variation leaves the Bayes classifier unchanged and increases the information required by log loss.

Finally, NLL and the conditional-KL gap are the primary measures of loss shift in this experiment. The Bayes-class representation can be calibrated relative to its coarsened information, because its optimal forecast is \(\mathbb E[\eta(X)\mid S]\). It is strictly less informative and strictly worse under proper scoring rules. The central empirical effect is a loss of probability-relevant information, measured exactly by NLL and the conditional-KL gap.

\subsection{Learned bottleneck representations}
\label{subsec:exp-learned-bottleneck}

The previous experiment constructs the Bayes-class representation explicitly. The next question is whether a learned representation can exhibit the same failure mode. The data-generating distribution is the same throughout source training, downstream training, and evaluation. The design removes distribution shift, so persistent downstream log-loss degradation relative to the full-input and fine-tuned controls is attributed to the frozen representation.

The same binary construction as in Section~\ref{subsec:exp-controlled} is used, with the representation learned from samples. Let
\[
X=(S,T),
\qquad
S\in\{-1,+1\},
\qquad
T\in\{1,\dots,m\},
\]
with \(m=10\). For a parameter \(\alpha\in[0,0.23]\), set
\[
\eta(+1,T)=0.75+\alpha r_T,
\qquad
\eta(-1,T)=0.25-\alpha r_T,
\]
where \(r_T\) ranges uniformly from \(-1\) to \(1\). For every value of \(\alpha\), the Bayes classifier is unchanged and depends on \(S\). As \(\alpha\) increases, the probability level \(\eta(S,T)\) varies more within each Bayes-class region, so log loss increasingly rewards information about \(T\).

Inputs are encoded as \((S,e_T,S e_T)\in\mathbb R^{21}\), where \(e_T\) is the one-hot encoding of \(T\). The encoder \(h_\theta\) is an MLP \(21\to64\to64\to1\), with layer normalization and \(\tanh\) activations after the two hidden layers. Source training uses a linear head on the bottleneck, minimizes the margin-one hinge loss \(\max\{0,1-(2Y-1)g(h_\theta(X))\}\), and injects Gaussian noise with standard deviation \(0.15\) at the bottleneck during training (variance \(0.0225\)). Together, the one-dimensional bottleneck and source-time noise create a controlled compressed-classification regime in which the theory predicts loss-shift effects.

For each \((\alpha,\text{seed})\), source training uses \(12{,}000\) samples, the downstream log-loss head uses an independent \(12{,}000\) samples, validation uses \(4{,}000\) samples, and testing uses \(50{,}000\) samples. All models are trained with batch size \(512\), AdamW, weight decay \(10^{-4}\), and gradient clipping at norm \(5\). The source encoder and frozen downstream head use learning rate \(2\times10^{-3}\); fine-tuning uses learning rate \(10^{-3}\). Training runs for at most \(350\) epochs, with validation-based early stopping using patience \(40\) and improvement threshold \(10^{-5}\); the best validation checkpoint is restored. After source training, the encoder is frozen and a nonlinear probabilistic head \(1\to64\to64\to1\) with \(\tanh\) activations is trained with binary log loss. This frozen transfer procedure is compared with two controls. The first is a full-input log-loss learner trained directly on \(X\) with a \(21\to128\to128\to1\) nonlinear head. The second fine-tunes the pretrained encoder and target head jointly under log loss, using zero representation noise. The fine-tuning control tests whether the probability information is present in \(X\) and unavailable through the frozen representation.

For reference, two oracle predictors are also reported. The oracle full predictor uses the true conditional probability \(\eta(X)\), and therefore gives the population target optimum. The oracle Bayes-class predictor uses only \(H=S\), and therefore gives the ideal classification-minimal representation. Its excess log loss is exactly
\[
\mathbb E\left[
\mathrm{kl}
\left(
\eta(S,T)
\,\middle\|\,
\mathbb E[\eta(S,T)\mid S]
\right)
\right].
\]
The experiment is repeated over \(30\) independent seeds, and means with \(95\%\) confidence-interval half-widths are reported.

\begin{table*}[t]
\centering
\caption{
Learned bottleneck experiment at the largest within-class probability variation, \(\alpha=0.23\). The learned hinge bottleneck preserves nearly Bayes-optimal classification accuracy and its frozen log-loss head has higher NLL than the full-input learner and the fine-tuned model. The \(\pm\) values are \(95\%\) confidence-interval half-widths computed over \(30\) independent seeds.
}
\label{tab:learned-bottleneck-main}
\begin{tabular}{lcccc}
\toprule
Method & Accuracy & NLL & Brier & ECE \\
\midrule
Oracle \(\eta(X)\)
& \(0.750 \pm 0.001\)
& \(0.499 \pm 0.001\)
& \(0.166 \pm 0.001\)
& \(0.005 \pm 0.001\) \\
Oracle Bayes class \(S\)
& \(0.750 \pm 0.001\)
& \(0.562 \pm 0.001\)
& \(0.187 \pm 0.001\)
& \(0.002 \pm 0.001\) \\
Full-input log-loss learner
& \(0.749 \pm 0.001\)
& \(0.499 \pm 0.002\)
& \(0.166 \pm 0.001\)
& \(0.012 \pm 0.002\) \\
Frozen hinge bottleneck
& \(0.749 \pm 0.002\)
& \(0.516 \pm 0.005\)
& \(0.171 \pm 0.002\)
& \(0.033 \pm 0.008\) \\
Fine-tuned hinge bottleneck
& \(0.749 \pm 0.001\)
& \(0.500 \pm 0.001\)
& \(0.167 \pm 0.001\)
& \(0.013 \pm 0.002\) \\
\bottomrule
\end{tabular}
\end{table*}

Table~\ref{tab:learned-bottleneck-main} shows the learned-representation result at \(\alpha=0.23\), where the log-loss-relevant variation inside each Bayes-class region is largest. All nontrivial methods have essentially the same classification accuracy, approximately \(0.75\), which is the Bayes accuracy of the construction. Accuracy therefore treats the full input, the frozen learned bottleneck, and the fine-tuned model as equivalent.

The target log loss separates them. The full-input learner achieves NLL \(0.499\), matching the oracle full predictor. The fine-tuned model also achieves NLL \(0.500\), showing that the probability information required by log loss is present in \(X\) and can be recovered when the encoder is allowed to change. The frozen hinge bottleneck achieves NLL \(0.516\), corresponding to an additional penalty of approximately \(0.017\) nats relative to the full-input learner. Thus the learned bottleneck retains enough information for the source classification decision and remains incomplete for optimal probabilistic prediction after freezing.

The oracle Bayes-class predictor \(S\) gives the limiting case of complete classification-minimal compression. At \(\alpha=0.23\), its NLL is \(0.562\), with the same accuracy as the full predictor. The learned bottleneck lies between the full-input predictor and this oracle class quotient. This is the expected ordering. The learned representation can retain more than the exact minimal quotient, and the compression induced by the source classification objective is sufficient to create a measurable frozen-transfer penalty under log loss.

\begin{figure}[t]
\centering
\includegraphics[width=.48\textwidth]{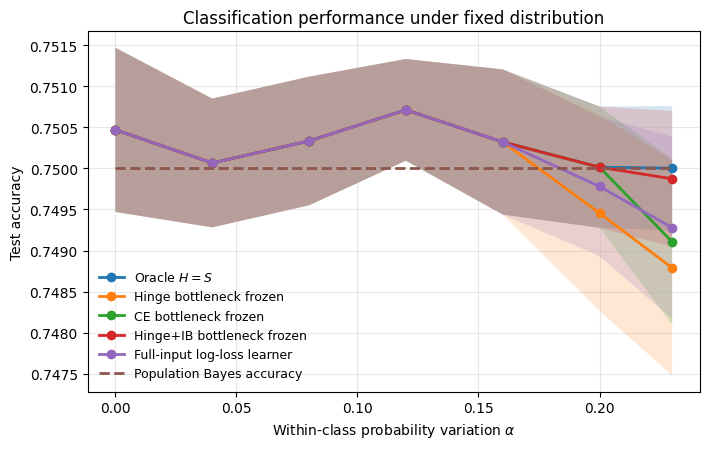}
\caption{
Classification accuracy in the learned bottleneck experiment. The Bayes classifier is unchanged throughout the sweep over \(\alpha\), and the learned frozen bottleneck maintains essentially the same accuracy as the full-input and fine-tuned models. The downstream log-loss obstruction is therefore hidden from the source classification objective.
}
\label{fig:learned-accuracy-alpha}
\end{figure}

\begin{figure}[t]
\centering
\includegraphics[width=.48\textwidth]{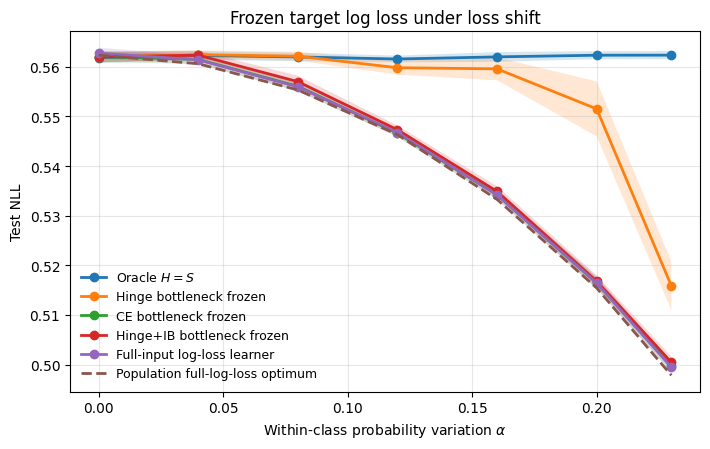}
\caption{
Negative log likelihood in the learned bottleneck experiment. The full-input log-loss learner and the fine-tuned model remain close to the oracle target risk. The frozen hinge bottleneck incurs a larger NLL as within-class probability variation becomes relevant to the target loss.
}
\label{fig:learned-nll-alpha}
\end{figure}

Figures~\ref{fig:learned-accuracy-alpha} and~\ref{fig:learned-nll-alpha} show the effect across the full \(\alpha\)-sweep. Increasing \(\alpha\) preserves the Bayes decision rule, so the classification problem remains the same from the perspective of the source loss. The downstream NLL of the frozen bottleneck separates from the full-input and fine-tuned controls. This is the learned analogue of the controlled result in Section~\ref{subsec:exp-controlled}. The obstruction is invisible to accuracy and visible to the target proper loss.

\begin{figure}[t]
\centering
\includegraphics[width=.48\textwidth]{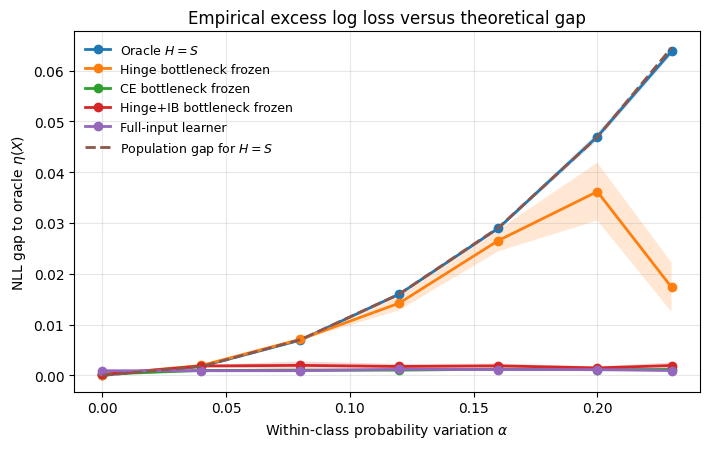}
\caption{
Excess NLL relative to the oracle predictor \(\eta(X)\). The oracle Bayes-class representation \(S\) follows the exact conditional-KL gap. The learned frozen bottleneck has a smaller positive gap, indicating that it preserves some probability-relevant information and remains insufficient for optimal log-loss transfer.
}
\label{fig:learned-nll-gap-alpha}
\end{figure}

Figure~\ref{fig:learned-nll-gap-alpha} compares the learned frozen-transfer gap with the exact oracle gap for \(H=S\). The oracle class representation tracks the conditional-KL expression predicted by Theorem~\ref{thm:exact-log-loss-gap}. The learned bottleneck has a smaller gap, because a one-dimensional real-valued representation can still retain some within-class information. Its gap is clearly separated from the full-input and fine-tuned controls at large \(\alpha\). This supports the central claim. Once a frozen representation has discarded probability-relevant variation, the target log-loss optimum is unavailable through downstream heads trained on that representation.

\begin{figure}[t]
\centering
\includegraphics[width=.48\textwidth]{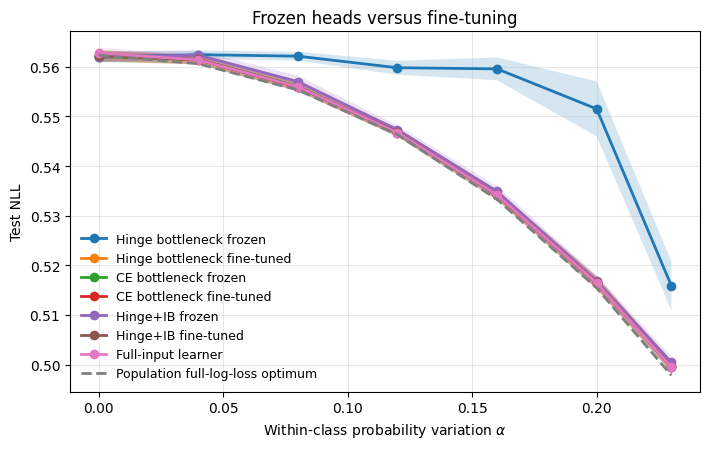}
\caption{
Frozen transfer versus fine-tuning. Fine-tuning the pretrained bottleneck under the target log loss recovers the full-input performance. Training a downstream head on the frozen bottleneck leaves a positive NLL gap. This shows that the target information is present in the input distribution and unavailable through the frozen representation.
}
\label{fig:learned-frozen-vs-finetuned}
\end{figure}

Figure~\ref{fig:learned-frozen-vs-finetuned} isolates the role of freezing. The fine-tuned model nearly matches the full-input log-loss learner, and the frozen model retains a positive gap. The degradation is explained by the information available through the frozen representation.

These results complement the exact construction in Section~\ref{subsec:exp-controlled}. The controlled experiment shows that a Bayes-minimal classification representation has a log-loss gap equal to the conditional information it discards. The learned experiment shows that the same mechanism can arise from optimization with a compressed classification representation. The learned bottleneck retains more information than the oracle quotient \(S\), giving a smaller gap than the exact Bayes-class gap, and it exhibits the same qualitative signature of unchanged classification accuracy, worse frozen log loss, and recovery when the representation is allowed to adapt to the target loss.

\subsection{dSprites image-based loss-shift experiment}
\label{subsec:exp-image-loss-shift}

An image-based experiment on dSprites tests the same accuracy-to-log-loss mechanism in an image setting where the relevant Bayes quotients are controlled through known latent factors. Five scale levels \(T\in\{0,\ldots,4\}\) are used, and a binary factor \(S\in\{-1,+1\}\) is defined by thresholding the horizontal position at the midpoint. The stochastic target is generated by
\[
Y\mid X \sim \operatorname{Bernoulli}(\eta(S,T)).
\]
At the main value \(\alpha=0.20\), the construction uses
\[
\eta(+1,T)=0.75+\alpha r_T,
\qquad
\eta(-1,T)=0.25-\alpha r_T,
\]
where \(r_T\) ranges uniformly from \(-1\) to \(1\) over the five selected scale levels. Thus the factor \(S\) determines the Bayes class, and \(T\) changes the conditional probability while preserving the sign of \(\eta(S,T)-1/2\). Consequently,
\[
Q_{0/1}(X)=S,
\qquad
Q_{\log}(X)=(S,T).
\]
As in the controlled model, classification requires \(S\), and log loss rewards preserving the probability-relevant factor \(T\).

The following five source-training conditions are compared.
\[
\begin{gathered}
\text{hinge } d=2,\quad
\text{hinge } d=8,\quad
\text{hinge } d=32,\\
\text{log-loss } d=8,\quad
\text{log-loss } d=32.
\end{gathered}
\]
Both hinge and log-loss encoders are trained on samples \((X,Y)\) from the same stochastic binary target; the source objectives receive \(X\) and \(Y\), with \(S\), \(T\), and \(\eta(S,T)\) used only for analysis. The hinge conditions include low-dimensional noisy bottlenecks and a higher-dimensional bottleneck, while the log-loss conditions serve as proper-loss source-training controls at matched dimensions. After source training, the encoder is frozen and a new binary log-loss head is trained on top of the frozen representation. Test accuracy, negative log likelihood (NLL), Brier score, expected calibration error (ECE), and diagnostic probe accuracies are reported for the latent factors \(S\), \(T\), and shape. Results are averaged over three seeds and reported as means \(\pm 1.96\) standard errors. Exact noise settings, additional implementation details, oracle metrics across the full \(\alpha\)-sweep, and calibration diagnostics are reported in Appendix~\ref{app:image-loss-shift}.

\paragraph{Oracle comparison.}
The oracle comparison isolates the loss-shift mechanism. The full posterior oracle uses \(\eta(S,T)\), and the Bayes-class oracle uses \(S\). For each fixed value of \(\alpha\), these oracles have the same accuracy because they induce the same Bayes decisions. The exact KL gap between the \(S\)-only predictor and the full posterior grows with \(\alpha\), showing that \(T\) becomes increasingly important for probability prediction while remaining irrelevant for classification. At \(\alpha=0.20\), the exact conditional-KL gap is \(0.059\) nats, and the empirical finite-test NLL difference between the two oracles is \(0.063\) nats. This trend is shown in Figure~\ref{fig:dsprites-main-results}, left. The complete oracle metrics across the \(\alpha\)-sweep are reported in Appendix~\ref{app:image-loss-shift}.

\paragraph{Learned frozen representations.}
Figure~\ref{fig:dsprites-main-results} summarizes the main dSprites results, and Table~\ref{tab:dsprites-learned-main} reports the learned frozen-encoder metrics at \(\alpha=0.20\). The same qualitative pattern appears in the learned models. All representations retain the classification factor \(S\) almost perfectly, with \(S\)-probe accuracies between \(0.982\) and \(0.990\), and therefore obtain similar target accuracies. The source objective affects how much probability-relevant information remains decodable from the frozen representation. The diagnostic probes recover \(T\) more accurately from log-loss encoders than from the corresponding hinge encoders. For example, the \(d=8\) log-loss encoder obtains \(T\)-probe accuracy \(0.485\), compared with \(0.428\) for the \(d=8\) hinge encoder. This difference is reflected in proper-loss performance. The \(d=8\) log-loss encoder achieves NLL \(0.524\) and Brier score \(0.174\), improving over both the \(d=8\) hinge encoder and the \(S\)-only Bayes-class oracle reported in Appendix~\ref{app:image-loss-shift}.

\begin{figure*}[t]
\centering
\includegraphics[width=.31\textwidth]{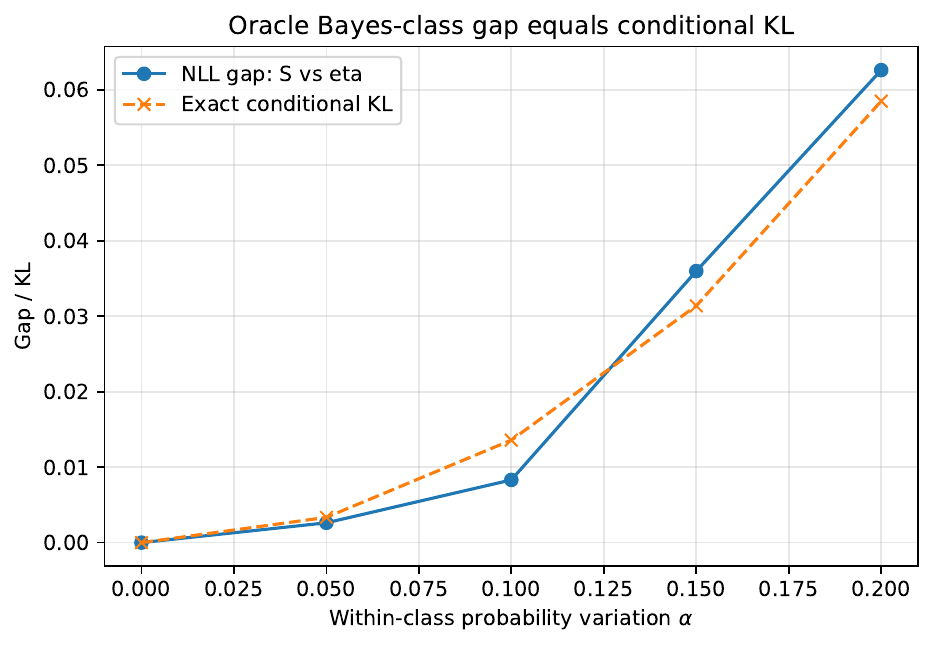}
\includegraphics[width=.31\textwidth]{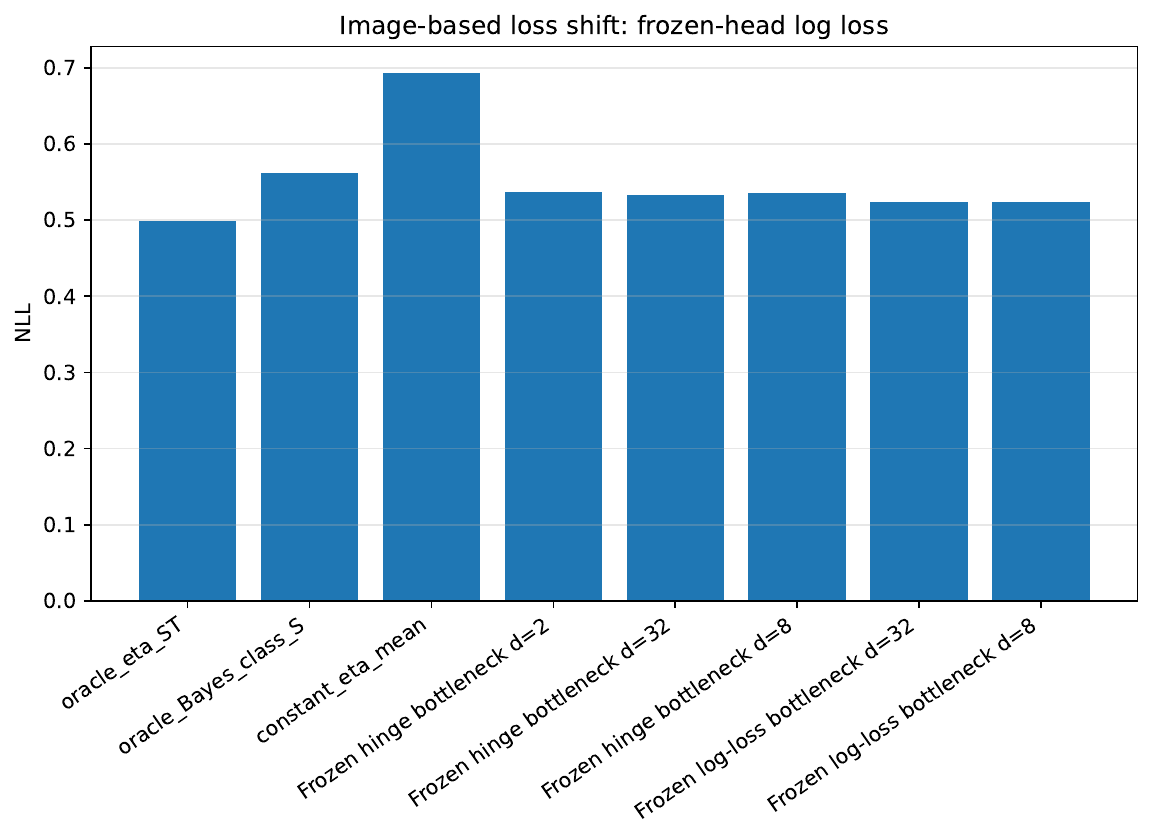}
\includegraphics[width=.31\textwidth]{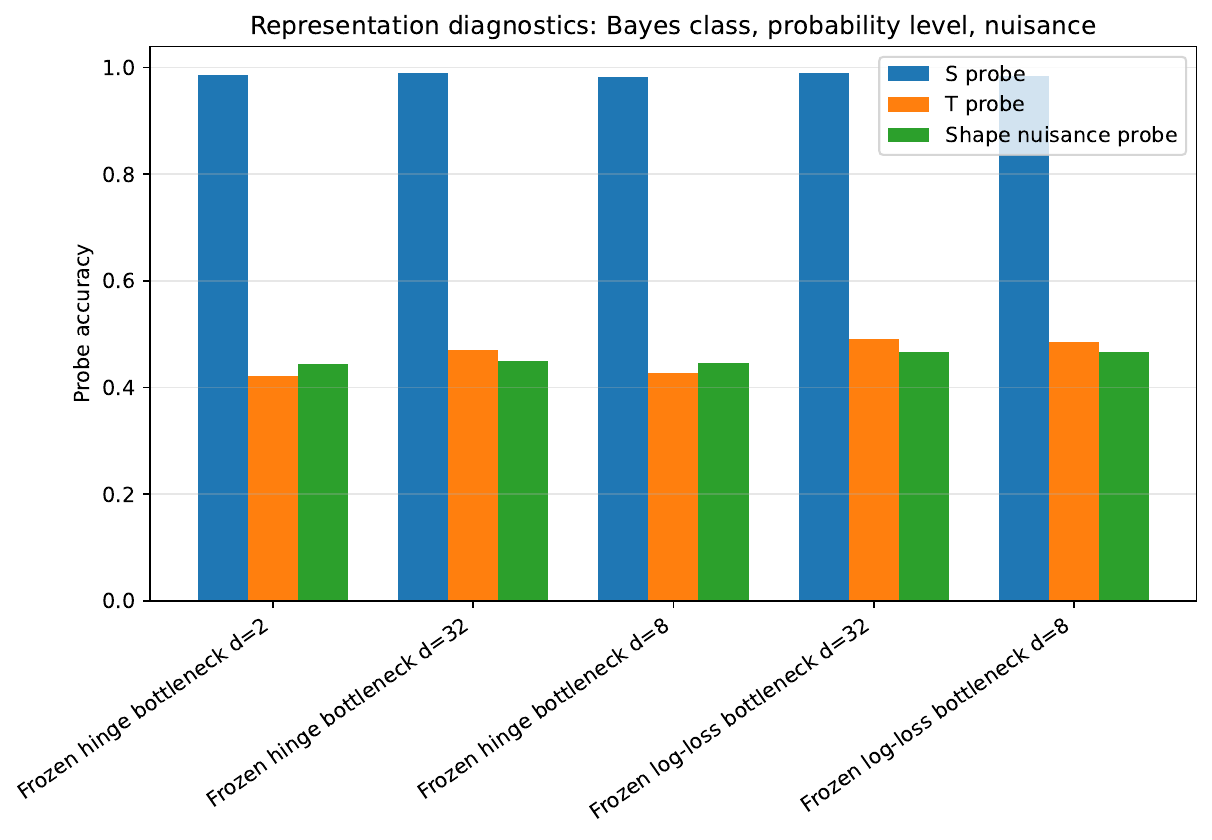}
\caption{
Main dSprites loss-shift results. Left panel, the exact KL gap between the full posterior oracle \(\eta(S,T)\) and the \(S\)-only Bayes-class oracle increases with \(\alpha\), confirming that \(T\) becomes increasingly relevant for proper-loss prediction while remaining irrelevant for the Bayes decision rule. Middle panel, at \(\alpha=0.20\), target NLL separates the full posterior oracle, the learned frozen representations, and the \(S\)-only Bayes-class oracle. Right panel, diagnostic probes show that all learned representations preserve the classification factor \(S\), and log-loss-trained representations make more information about the probability-level factor \(T\) decodable.
}
\label{fig:dsprites-main-results}
\end{figure*}

\begin{table*}[t]
\centering
\caption{
Learned frozen-representation results on dSprites at \(\alpha=0.20\). Encoders are trained under the source objective shown in the first column, then frozen; a new binary log-loss head is trained on top of the frozen representation. The probe columns report one-hidden-layer diagnostic probe accuracies for the latent factors. The \(\pm\) values denote \(1.96\) standard errors over three seeds. All methods preserve the classification factor \(S\), and log-loss-trained representations make more information about the probability-level factor \(T\) decodable, leading to lower target NLL and Brier score.
}
\label{tab:dsprites-learned-main}
\resizebox{\textwidth}{!}{
\begin{tabular}{lccccccc}
\toprule
Method & Acc. & NLL & Brier & ECE & \(S\) probe & \(T\) probe & Shape probe \\
\midrule
Frozen hinge, \(d=2\)
& \(0.743 \pm 0.001\)
& \(0.536 \pm 0.002\)
& \(0.179 \pm 0.001\)
& \(0.035 \pm 0.007\)
& \(0.986 \pm 0.004\)
& \(0.422 \pm 0.022\)
& \(0.445 \pm 0.012\) \\
Frozen hinge, \(d=8\)
& \(0.742 \pm 0.001\)
& \(0.536 \pm 0.002\)
& \(0.179 \pm 0.001\)
& \(0.032 \pm 0.003\)
& \(0.982 \pm 0.010\)
& \(0.428 \pm 0.021\)
& \(0.445 \pm 0.004\) \\
Frozen hinge, \(d=32\)
& \(0.743 \pm 0.002\)
& \(0.533 \pm 0.008\)
& \(0.177 \pm 0.003\)
& \(0.033 \pm 0.004\)
& \(0.989 \pm 0.001\)
& \(0.471 \pm 0.036\)
& \(0.449 \pm 0.014\) \\
Frozen log-loss, \(d=8\)
& \(0.744 \pm 0.002\)
& \(0.524 \pm 0.003\)
& \(0.174 \pm 0.002\)
& \(0.024 \pm 0.002\)
& \(0.984 \pm 0.008\)
& \(0.485 \pm 0.010\)
& \(0.467 \pm 0.002\) \\
Frozen log-loss, \(d=32\)
& \(0.745 \pm 0.002\)
& \(0.524 \pm 0.001\)
& \(0.174 \pm 0.000\)
& \(0.028 \pm 0.008\)
& \(0.990 \pm 0.001\)
& \(0.490 \pm 0.013\)
& \(0.466 \pm 0.005\) \\
\bottomrule
\end{tabular}
}
\end{table*}

These results support the same interpretation as the controlled and learned bottleneck experiments. The \(S\)-probe accuracies are close to one for all source-training conditions, confirming that all methods preserve the factor needed for classification. The \(T\)-probe accuracies are substantially above chance and are higher for the log-loss representations than for the compressed hinge representations. Since \(T\) is irrelevant to the Bayes class and relevant to \(\eta(S,T)\), this diagnostic pattern supports the interpretation that proper-loss source training preserves probability-relevant information that classification-style source training can compress away. The downstream NLL and Brier scores follow the same ordering. Log-loss-trained frozen representations perform better under the target proper loss than compressed hinge representations with comparable classification accuracy.

\subsection{CIFAR-10H human-label experiment}
\label{subsec:exp-cifar10h}

The final experiment tests the same mechanism on real images with human soft labels. CIFAR-10H provides an empirical distribution \(p_i\in\Delta^{10}\) of human labels for each image in the CIFAR-10 test set \citep{krizhevsky2009learning,peterson2019human}. The input images are fixed throughout. Viewing a human response as
\[
Y\mid X=x_i \sim p_i,
\]
the source classification quotient is the majority label
\[
y_i^{\mathrm{maj}}=\arg\max_y p_i(y),
\]
whereas the log-loss target requires the full vector \(p_i\). This gives a real-image analogue of the accuracy-to-probability-estimation setting. The same image can have an unambiguous majority label while still carrying human uncertainty that is relevant to the target loss.

Models are trained on a fixed split of the \(10{,}000\) CIFAR-10H images. A hard-label source encoder is trained with cross-entropy against \(y_i^{\mathrm{maj}}\), then frozen and equipped with a new soft-label head trained against \(p_i\). This is compared with temperature scaling of the hard classifier, hard-label bottlenecks of dimensions \(d\in\{2,8,32,128\}\), a soft-label source encoder trained directly against \(p_i\), and a hard-label encoder fine-tuned under soft-label log loss. Two oracles are included, namely the human-distribution oracle \(p_i\) and the majority-class quotient oracle that predicts the average human distribution conditional on \(y_i^{\mathrm{maj}}\). Results are averaged over three seeds for learned models; additional implementation details and the full table are reported in Appendix~\ref{app:cifar10h}.

\begin{table*}[t]
\centering
\caption{
CIFAR-10H real-image loss-shift experiment. Accuracy is measured against the majority human label, while NLL and KL are evaluated against the full human label distribution. Learned-model intervals are mean \(\pm 1.96\) standard errors over three seeds. The soft-label encoder gives substantially lower soft-label NLL and KL than the hard-label frozen encoder, despite similar majority-label accuracy.
}
\label{tab:cifar10h-main}
\resizebox{\textwidth}{!}{
\begin{tabular}{lccc}
\toprule
Method & Majority-label acc. & Soft-label NLL & KL to human labels \\
\midrule
Human distribution oracle
& 1.000 & 0.147 & 0.000 \\
Majority-class quotient oracle
& 1.000 & 0.243 & 0.095 \\
Hard encoder + soft head
& \(0.701 \pm 0.040\)
& \(0.936 \pm 0.097\)
& \(0.789 \pm 0.097\) \\
Hard encoder + temperature scaling
& \(0.674 \pm 0.068\)
& \(0.994 \pm 0.150\)
& \(0.847 \pm 0.150\) \\
Hard bottleneck \(d=2\)
& \(0.405 \pm 0.119\)
& \(1.515 \pm 0.213\)
& \(1.367 \pm 0.213\) \\
Hard bottleneck \(d=8\)
& \(0.680 \pm 0.024\)
& \(1.003 \pm 0.073\)
& \(0.856 \pm 0.073\) \\
Soft encoder + soft head
& \(0.765 \pm 0.025\)
& \(0.769 \pm 0.052\)
& \(0.622 \pm 0.052\) \\
Hard encoder fine-tuned
& \(0.753 \pm 0.047\)
& \(0.870 \pm 0.057\)
& \(0.722 \pm 0.057\) \\
\bottomrule
\end{tabular}
}
\end{table*}

Table~\ref{tab:cifar10h-main} shows the main effect. The majority-class quotient oracle has the same majority-label accuracy as the human-distribution oracle, but incurs an additional KL of \(0.095\) nats because it discards within-class human uncertainty. Among learned representations, the hard-label frozen encoder and the soft-label frozen encoder have comparable majority-label accuracies, but the soft-label encoder is substantially better under the target proper loss. NLL decreases from \(0.936\) to \(0.769\), and KL to the human distribution decreases from \(0.789\) to \(0.622\). Temperature scaling does not close this gap. Fine-tuning the hard encoder under the target loss improves NLL relative to the frozen hard encoder, indicating that the limitation is tied to freezing rather than absence of signal in the images.

\begin{figure*}[t]
\centering
\includegraphics[width=.70\textwidth]{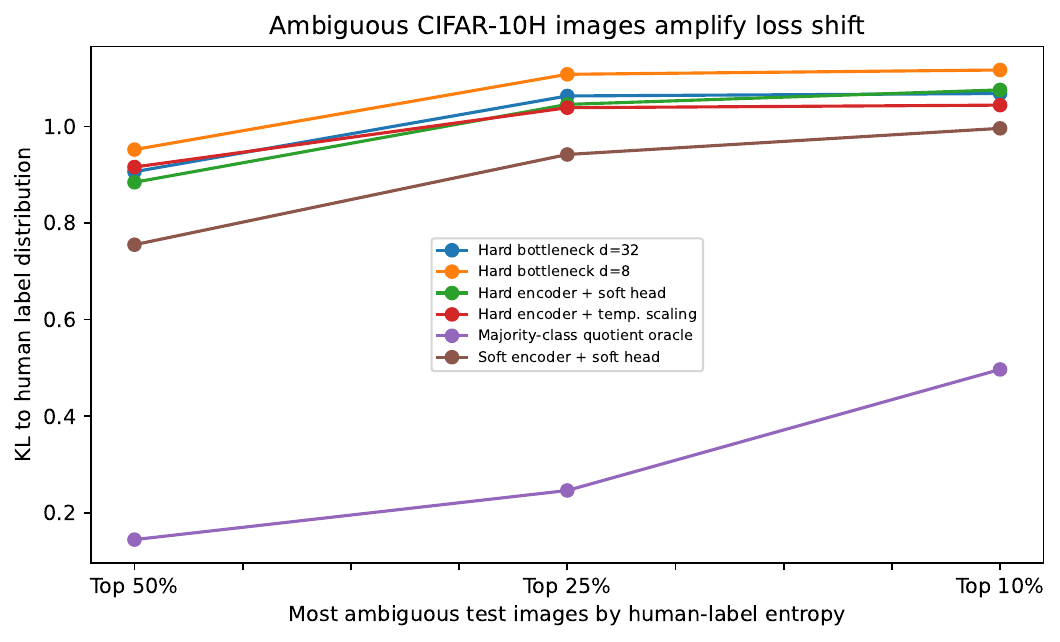}
\caption{
CIFAR-10H entropy-stratified loss-shift diagnostic. Restricting evaluation to the most ambiguous images increases the oracle majority-quotient gap and amplifies the learned soft-label loss differences. The soft-label source encoder remains the best learned frozen representation across entropy subsets.
}
\label{fig:cifar10h-main}
\end{figure*}

Figure~\ref{fig:cifar10h-main} further isolates the role of human-label ambiguity. On the top \(10\%\) highest-entropy images, the majority-class quotient oracle incurs KL \(0.497\), compared with \(0.144\) on the top \(50\%\). The learned models show the same qualitative pressure. Ambiguous images are where soft-label prediction is most difficult, and source training against the full human distribution gives the best frozen representation among the learned encoders. The bottleneck sweep provides a complementary compression check. The \(d=2\) hard bottleneck collapses both accuracy and soft-label prediction, while \(d=8,32,128\) recover much of the majority-label accuracy but remain inferior to the soft-label encoder under KL and NLL; the full sweep is reported in Appendix~\ref{app:cifar10h}. This is the real-image counterpart of the earlier quotient story. Source training for a hard classification objective can preserve top-label information while losing probability-relevant distinctions needed for a finer target loss.

\section{Limitations}
\label{sec:limitations}

The results isolate a specific representation-level obstruction to frozen transfer. The lower bound applies when the target loss strictly refines the source-loss Bayes quotient and the frozen representation is source-minimal, or compressed toward that quotient. This does not imply that every source-trained representation fails; representations that retain target-relevant information can transfer successfully. The dSprites and CIFAR-10H experiments use this contrast directly, comparing compressed classification representations with log-loss-trained representations that preserve more probability-relevant information.

The formal theory is stated in the Bayes-quotient setting. This setting covers the standard unique-Bayes-action cases used throughout the paper, including zero-one loss without ties and finite-label log loss. Losses with nonunique Bayes actions require a set-valued or common-action formulation of the quotient. This extension is left to future work.

The exact information-theoretic identity is specialized to finite-output log loss. This case is central for probabilistic classification and calibration. Other target losses may require different regret decompositions. Proper scoring rules, Bregman losses, quantile losses, and structured prediction losses may admit analogous quotient refinements, with excess-risk expressions that take forms beyond conditional mutual information.

The experiments are designed to isolate the mechanism rather than estimate its prevalence. The first two settings are controlled models, the dSprites experiment uses known image factors to fix \(P(X,Y)\) while changing the training and evaluation loss, and CIFAR-10H uses fixed images with human soft-label distributions. Quantifying how often loss shift appears in large-scale pretraining is a broader empirical question. The framework here provides a diagnostic. When a frozen representation is compressed relative to the target-loss quotient, downstream heads cannot recover the discarded target-relevant information.

Finally, by optimizing over all measurable heads on a frozen representation, the theory isolates the representation-level component of transfer failure from finite-sample, head-class, and optimization effects. These effects can coexist in practice, but a positive measurable-head gap is a representation-level transfer penalty that persists even with unrestricted downstream head capacity for the fixed representation.

\section{Conclusion}
\label{sec:conclusion}

This paper studied transfer failure under a fixed data distribution. The joint law \(P(X,Y)\) was held fixed throughout, and the predictive loss changed. In this setting, transfer failure arises when the target loss asks for information beyond what the source loss required.

The main object was the loss-induced Bayes quotient. For a fixed distribution and loss, this quotient captures exactly the information in \(X\) needed for Bayes-optimal prediction. It provides a representation-level notion of Bayes-relevant information and allows losses to be ordered by refinement. Under strict refinement, a representation that is minimal for the source loss is insufficient for the target loss. The obstruction holds even when the downstream head ranges over all measurable predictors.

For log loss on finite output spaces, the lower bound becomes exact. The frozen-transfer excess risk is
\[
R^\star_{\log,P}(H)-R^\star_{\log,P}(X)
=
I(Y;X\mid H),
\]
or equivalently the expected conditional KL divergence between \(P(Y\mid X)\) and \(P(Y\mid H)\). Thus the log-loss transfer penalty is precisely the predictive information about \(Y\) discarded by the frozen representation. The binary classification-to-log-loss case makes the mechanism transparent. Accuracy requires the Bayes class, and log loss requires the conditional probability.

The experiments support this interpretation. In the controlled model, the Bayes-class representation achieves the same classification accuracy as the full representation and incurs a positive log-loss gap that matches the exact conditional-KL prediction. In the learned bottleneck experiment, a compressed classification representation exhibits the same signature of unchanged accuracy, worse frozen log loss, and recovery when the encoder is fine-tuned under the target loss. In the dSprites construction, compressed hinge representations retain the classification-relevant factor and preserve less of the probability-relevant scale factor than log-loss representations. In CIFAR-10H, hard-label frozen encoders and compressed bottlenecks remain worse than soft-label encoders at predicting human label distributions, especially on ambiguous images.

The broader message is that the usefulness of a representation is determined by the pair \((P,\ell)\). A representation can be optimal for one loss and intrinsically insufficient for another, even when source and target examples are drawn from the same distribution. Loss shift therefore complements distribution shift as a basic mechanism in transfer learning. It identifies when frozen transfer fails because the objective has changed the meaning of the information that must be preserved.

\bibliography{references}

\appendix

\section{Additional details for the image-based loss-shift experiment}
\label{app:image-loss-shift}

\paragraph{Dataset and splits.}
The experiment uses the dSprites dataset with five scale levels \(T\in\{0,\ldots,4\}\), and the binary factor \(S\) is defined by thresholding the horizontal position at the midpoint. Images are \(64\times64\) binary arrays encoded as one-channel floating point tensors. The train, validation, and test splits contain 800, 160, and 320 examples per \((S,T)\) cell, respectively, giving 8{,}000 source/downstream training examples, 1{,}600 validation examples, and 3{,}200 test examples at each value of \(\alpha\). The balanced subset is sampled with seed 12345. Labels are sampled once with label seed 777. Unless otherwise stated, results use training seeds \(0,1,2\).

\paragraph{Latent factors.}
The factor \(S\) determines the Bayes class. The factor \(T\) changes the conditional probability \(\eta(S,T)\) while preserving the sign of \(\eta(S,T)-1/2\). Consequently,
\[
Q_{0/1}(X) = S,
\qquad
Q_{\log}(X) = (S,T).
\]
The shape factor is excluded from the label-generation mechanism and included as an additional diagnostic probe.

\paragraph{Training conditions.}
The following five source-training conditions are compared at the main value \(\alpha=0.20\).
\[
\begin{gathered}
\text{hinge } d=2,\quad
\text{hinge } d=8,\quad
\text{hinge } d=32,\\
\text{log-loss } d=8,\quad
\text{log-loss } d=32.
\end{gathered}
\]
The hinge \(d=2\) and \(d=8\) settings use i.i.d. Gaussian representation noise \(\mathcal N(0,0.2^2 I)\) during source training. The \(d=32\) hinge and log-loss settings use zero representation noise. The hinge source loss uses margin one, and the log-loss source loss uses binary cross-entropy with logits. Encoders are trained for 18 epochs. After source training, the encoder is frozen and a new binary log-loss head is trained for 12 epochs. Diagnostic probes for \(S\), \(T\), and shape are trained for 10 epochs.

\paragraph{Architectures and optimization.}
The encoder is a four-block convolutional network with channel widths \(16,32,64,64\), kernel size \(4\), stride \(2\), padding \(1\), and ReLU activations, followed by a 128-unit ReLU layer and a linear bottleneck of dimension \(d\). The source head is linear on the bottleneck. The frozen downstream head is a two-hidden-layer MLP \(d\to64\to64\to1\) with ReLU activations. Each diagnostic probe is a one-hidden-layer MLP \(d\to h\to k\), where \(h=\max(32,2d)\), \(k=1\) for the binary \(S\) probe, \(k=5\) for \(T\), and \(k=3\) for shape. Training uses AdamW with learning rate \(10^{-3}\), weight decay \(10^{-4}\), and batch size 256 for source encoders, downstream heads, and probes. Source encoders and downstream heads are selected by validation NLL; probes are selected by validation accuracy. ECE is computed with 15 equal-width bins.

\paragraph{Oracle metrics across \(\alpha\).}
Table~\ref{tab:oracle-alpha-sweep} reports oracle performance across the alpha sweep. The two oracles have identical accuracy within each value of \(\alpha\), because they induce the same Bayes decisions. Small changes in the reported accuracy across \(\alpha\) reflect the sampled stochastic labels and finite test set. The exact KL gap between the \(S\)-only predictor and the full posterior increases with \(\alpha\), showing that \(T\) becomes increasingly relevant for probability prediction.

\begin{table*}[t]
\centering
\caption{
Oracle metrics across \(\alpha\). The full posterior oracle uses \(\eta(S,T)\); the Bayes-class oracle uses only \(S\). The exact KL gap is measured relative to the full posterior oracle.
}
\label{tab:oracle-alpha-sweep}
\begin{tabular}{llccccc}
\toprule
\(\alpha\) & Oracle & Acc. & NLL & Brier & ECE & KL gap \\
\midrule
0.00 & \(\eta(S,T)\) & 0.759 & 0.552 & 0.183 & 0.009 & 0.000 \\
0.00 & \(S\)-only & 0.759 & 0.552 & 0.183 & 0.009 & 0.000 \\
0.05 & \(\eta(S,T)\) & 0.744 & 0.566 & 0.189 & 0.007 & 0.000 \\
0.05 & \(S\)-only & 0.744 & 0.569 & 0.190 & 0.006 & 0.003 \\
0.10 & \(\eta(S,T)\) & 0.753 & 0.551 & 0.183 & 0.023 & 0.000 \\
0.10 & \(S\)-only & 0.753 & 0.559 & 0.186 & 0.007 & 0.014 \\
0.15 & \(\eta(S,T)\) & 0.733 & 0.545 & 0.183 & 0.022 & 0.000 \\
0.15 & \(S\)-only & 0.733 & 0.581 & 0.196 & 0.017 & 0.031 \\
0.20 & \(\eta(S,T)\) & 0.751 & 0.499 & 0.166 & 0.019 & 0.000 \\
0.20 & \(S\)-only & 0.751 & 0.562 & 0.187 & 0.006 & 0.059 \\
\bottomrule
\end{tabular}
\end{table*}

\paragraph{Calibration note.}
Empirical ECE can favor coarser or less variable predictors depending on the binning and sample realization, because it is a binned finite-sample diagnostic. The loss-aligned oracle comparison is clearest in NLL, Brier score, and the exact KL gap, which directly measure the information lost by replacing \(\eta(S,T)\) with an \(S\)-only predictor.

\paragraph{Interpretation of probes.}
The \(S\)-probe accuracies are close to one for all learned representations, confirming that all methods preserve the factor needed for classification. The \(T\)-probe accuracies are substantially above chance and are higher for the log-loss representations than for the compressed hinge representations. Since \(T\) is irrelevant to the Bayes class and relevant to \(\eta(S,T)\), this supports the interpretation that proper-loss source training preserves probability-relevant information that classification-style source training can compress away.

\paragraph{Additional figures.}
The following figures are included in the supplementary material.
\begin{itemize}
    \item \texttt{image\_loss\_shift\_oracle\_exact\_kl.pdf} reports the exact oracle KL gap across \(\alpha\).
    \item \texttt{image\_loss\_shift\_alpha\_sweep\_gap.pdf} reports frozen-representation target-log-loss behavior across \(\alpha\).
    \item \texttt{image\_loss\_shift\_accuracy\_bar.pdf} gives the accuracy comparison at \(\alpha=0.20\).
    \item \texttt{image\_loss\_shift\_nll\_bar.pdf} gives the NLL comparison at \(\alpha=0.20\).
    \item \texttt{image\_loss\_shift\_probe\_bar.pdf} gives the diagnostic-probe accuracies at \(\alpha=0.20\).
\end{itemize}

\section{Additional details for the CIFAR-10H experiment}
\label{app:cifar10h}

\paragraph{Dataset and fixed distribution.}
The experiment uses the \(10{,}000\) CIFAR-10 test images together with the CIFAR-10H human label distributions \citep{krizhevsky2009learning,peterson2019human}. The CIFAR-10H probability vector for image \(i\) is denoted \(p_i\in\Delta^{10}\). The majority label is \(y_i^{\mathrm{maj}}=\arg\max_y p_i(y)\). The empirical distribution is fixed by sampling images uniformly from the CIFAR-10H image set and, for the target-log-loss interpretation, sampling a human response according to \(p_i\).

\paragraph{Splits and training.}
The \(10{,}000\) images are split once into \(60\%\) training, \(20\%\) validation, and \(20\%\) test using split seed \(123\). Learned-model results use training seeds \(0,1,2\). Source encoders are trained for at most \(80\) epochs, hard bottlenecks for at most \(60\) epochs, downstream soft-label heads for at most \(120\) epochs, and soft-label fine-tuning for at most \(60\) epochs, with validation-based early stopping using patience \(15\). Training uses AdamW, learning rate \(10^{-3}\) for source encoders and downstream heads, weight decay \(10^{-4}\), batch size \(128\), random crop and horizontal flip during source training, and standard CIFAR-10 normalization.

\paragraph{Architectures.}
The base encoder is a small convolutional network with three convolutional stages of widths \(64,128,256\), batch normalization, ReLU activations, max pooling after the first two stages, adaptive average pooling, and a \(256\)-dimensional feature projection. A source classifier is linear on these features. Frozen soft heads are two-layer MLP heads on top of the frozen representation. The hard bottleneck conditions freeze the hard-label source encoder and train a ReLU bottleneck of dimension \(d\in\{2,8,32,128\}\) with a hard-label classifier; Gaussian representation noise with standard deviation \(0.05\) is used while training the bottleneck. After source training, the bottleneck is frozen and a soft-label head is trained against \(p_i\). Temperature scaling fits one scalar temperature on the validation split by minimizing soft-label cross-entropy.

\paragraph{Metrics.}
For a model prediction \(q_i\), the reported soft-label NLL is
\[
-\sum_y p_i(y)\log q_i(y),
\]
and KL to the human label distribution is
\[
D_{\mathrm{KL}}(p_i\|q_i)
=
-\sum_y p_i(y)\log q_i(y)
-\Ent(p_i).
\]
Accuracy is measured against \(y_i^{\mathrm{maj}}\). Brier score is \(\sum_y(q_i(y)-p_i(y))^2\). ECE is computed with \(15\) equal-width confidence bins against \(y_i^{\mathrm{maj}}\) and is treated as a calibration diagnostic rather than the loss-aligned target metric.

\paragraph{Entropy stratification.}
Human-label entropy is \(\Ent(p_i)=-\sum_y p_i(y)\log p_i(y)\). Since many CIFAR-10H rows are deterministic, ordinary quantile binning creates repeated entropy edges. The entropy-stratified table therefore uses three bins, i.e., \(H=0\), positive entropy below the positive-entropy median \(0.1705\), and entropy above that median. The top-entropy analysis in Figure~\ref{fig:cifar10h-main} additionally reports the top \(50\%\), \(25\%\), and \(10\%\) most ambiguous test images by \(\Ent(p_i)\).

\begin{table*}[t]
\centering
\caption{
Full CIFAR-10H metrics. Learned-model intervals are mean \(\pm 1.96\) standard errors over three seeds. Oracle rows are deterministic for the fixed test split.
}
\label{tab:cifar10h-full}
\resizebox{\textwidth}{!}{
\begin{tabular}{lccccc}
\toprule
Method & Accuracy & Soft NLL & KL & Brier & ECE \\
\midrule
Human distribution oracle
& 1.0000 & 0.1475 & 0.0000 & 0.0000 & 0.0441 \\
Majority-class quotient oracle
& 1.0000 & 0.2426 & 0.0952 & 0.0131 & 0.0456 \\
Train mean
& 0.1020 & 2.3027 & 2.1552 & 0.8278 & 0.0028 \\
Hard encoder + temperature scaling
& \(0.6738 \pm 0.0684\)
& \(0.9941 \pm 0.1497\)
& \(0.8467 \pm 0.1497\)
& \(0.3905 \pm 0.0716\)
& \(0.0317 \pm 0.0049\) \\
Hard encoder + soft head
& \(0.7007 \pm 0.0402\)
& \(0.9363 \pm 0.0971\)
& \(0.7888 \pm 0.0971\)
& \(0.3628 \pm 0.0480\)
& \(0.0270 \pm 0.0104\) \\
Hard bottleneck \(d=2\)
& \(0.4050 \pm 0.1185\)
& \(1.5149 \pm 0.2125\)
& \(1.3675 \pm 0.2125\)
& \(0.6237 \pm 0.0824\)
& \(0.0396 \pm 0.0179\) \\
Hard bottleneck \(d=8\)
& \(0.6797 \pm 0.0240\)
& \(1.0031 \pm 0.0727\)
& \(0.8557 \pm 0.0727\)
& \(0.3900 \pm 0.0316\)
& \(0.0280 \pm 0.0087\) \\
Hard bottleneck \(d=32\)
& \(0.6852 \pm 0.0499\)
& \(0.9641 \pm 0.1141\)
& \(0.8166 \pm 0.1141\)
& \(0.3762 \pm 0.0549\)
& \(0.0288 \pm 0.0032\) \\
Hard bottleneck \(d=128\)
& \(0.6937 \pm 0.0331\)
& \(0.9484 \pm 0.0964\)
& \(0.8009 \pm 0.0964\)
& \(0.3686 \pm 0.0448\)
& \(0.0335 \pm 0.0036\) \\
Soft encoder + soft head
& \(0.7652 \pm 0.0248\)
& \(0.7691 \pm 0.0516\)
& \(0.6216 \pm 0.0516\)
& \(0.2819 \pm 0.0291\)
& \(0.0258 \pm 0.0016\) \\
Hard encoder fine-tuned
& \(0.7528 \pm 0.0475\)
& \(0.8697 \pm 0.0572\)
& \(0.7222 \pm 0.0572\)
& \(0.3066 \pm 0.0438\)
& \(0.0763 \pm 0.0110\) \\
\bottomrule
\end{tabular}
}
\end{table*}

\paragraph{Additional CIFAR-10H figures.}
The following figures are included with the source.
\begin{itemize}
    \item \texttt{cifar10h\_accuracy\_vs\_kl.pdf} shows majority-label accuracy versus KL to human label distributions.
    \item \texttt{cifar10h\_soft\_nll\_bar.pdf} shows the soft-label NLL comparison.
    \item \texttt{cifar10h\_kl\_bar.pdf} shows the KL comparison.
    \item \texttt{cifar10h\_entropy\_stratified\_kl.pdf} shows KL across the \(H=0\), positive-mid-entropy, and high-entropy bins.
    \item \texttt{cifar10h\_entropy\_top\_kl.pdf} shows KL on the top \(50\%\), \(25\%\), and \(10\%\) most ambiguous images.
    \item \texttt{cifar10h\_bottleneck\_sweep\_kl.pdf} shows hard-label bottleneck dimension versus KL.
\end{itemize}

\section{Proofs}
\label{app:methodology-proofs}

\subsection{Proof of Proposition~\ref{prop:loss-preorder}}
\label{app:proof-loss-preorder}

\begin{proof}
Let $\ell$ be any loss admitting a Bayes quotient under $P$. Since every sigma-algebra is contained in itself,
\begin{equation}
\sigma(q_{\ell,P}(X))
\subseteq
\sigma(q_{\ell,P}(X)).
\end{equation}
Hence $\ell\preceq_P\ell$, so $\preceq_P$ is reflexive.

Now let $\ell_1,\ell_2,\ell_3$ be losses admitting Bayes quotients under $P$ and satisfying
\begin{equation}
\ell_1\preceq_P\ell_2
\qquad\text{and}\qquad
\ell_2\preceq_P\ell_3.
\end{equation}
By definition,
\begin{equation}
\sigma(q_{\ell_1,P}(X))
\subseteq
\sigma(q_{\ell_2,P}(X))
\end{equation}
and
\begin{equation}
\sigma(q_{\ell_2,P}(X))
\subseteq
\sigma(q_{\ell_3,P}(X)).
\end{equation}
By transitivity of set inclusion,
\begin{equation}
\sigma(q_{\ell_1,P}(X))
\subseteq
\sigma(q_{\ell_3,P}(X)).
\end{equation}
Therefore $\ell_1\preceq_P\ell_3$, so $\preceq_P$ is transitive.

Finally, on equivalence classes under $\sim_P$, antisymmetry holds by construction. If
\begin{equation}
[\ell_1]\preceq_P[\ell_2]
\qquad\text{and}\qquad
[\ell_2]\preceq_P[\ell_1],
\end{equation}
then
\begin{equation}
\sigma(q_{\ell_1,P}(X))
=
\sigma(q_{\ell_2,P}(X))
\qquad
\text{mod }P_X.
\end{equation}
Thus $\ell_1\sim_P\ell_2$, and hence $[\ell_1]=[\ell_2]$. Therefore $\preceq_P$ induces a partial order on equivalence classes.
\end{proof}

\subsection{Proof of Proposition~\ref{prop:loss-shift-obstruction}}
\label{app:proof-loss-shift-obstruction}

\begin{proof}
Since $\ell_1\prec_P\ell_2$, it follows that
\begin{equation}
\sigma(q_{\ell_1,P}(X))
\subseteq
\sigma(q_{\ell_2,P}(X))
\end{equation}
and
\begin{equation}
\sigma(q_{\ell_2,P}(X))
\not\subseteq
\sigma(q_{\ell_1,P}(X)).
\end{equation}
Since $H_1^\star$ is Bayes-minimal for $(P,\ell_1)$,
\begin{equation}
\sigma(H_1^\star)
=
\sigma(q_{\ell_1,P}(X))
\qquad
\text{mod }P_X.
\end{equation}
Therefore,
\begin{equation}
\sigma(q_{\ell_2,P}(X))
\not\subseteq
\sigma(H_1^\star).
\end{equation}
By the Bayes-quotient characterization recalled in Section~\ref{subsec:bayes-quotient-setting}, a representation $H$ is Bayes-sufficient for $(P,\ell_2)$ if and only if
\begin{equation}
\sigma(q_{\ell_2,P}(X))
\subseteq
\sigma(H).
\end{equation}
Applying this criterion to $H=H_1^\star$, the preceding non-inclusion implies that $H_1^\star$ is not Bayes-sufficient for $(P,\ell_2)$.
\end{proof}

\subsection{Proof of Theorem~\ref{thm:exact-log-loss-gap}}
\label{app:proof-exact-log-loss-gap}

\begin{proof}
Let $c:\mathcal H\to\Delta(\mathcal Y)$ be any measurable head. The log risk of the predictor $c(H)$ is
\begin{equation}
\mathbb E[-\log c(H)(Y)].
\end{equation}
Conditioning on $H$ gives
\begin{equation}
\mathbb E[-\log c(H)(Y)]
=
\mathbb E\left[
\mathbb E[-\log c(H)(Y)\mid H]
\right].
\end{equation}
For a fixed representation value $H=u$, the conditional law of $Y$ is $\pi_H(u)$. Hence
\begin{equation}
\mathbb E[-\log c(H)(Y)\mid H=u]
=
-\sum_{y\in\mathcal Y}
\pi_H(u)(y)\log c(u)(y).
\end{equation}
This cross-entropy decomposes as
\begin{equation}
-\sum_{y\in\mathcal Y}
\pi_H(u)(y)\log c(u)(y)
=
-\sum_{y\in\mathcal Y}
\pi_H(u)(y)\log \pi_H(u)(y)
+
D_{\mathrm{KL}}(\pi_H(u)\|c(u)).
\end{equation}
Since $D_{\mathrm{KL}}(\pi_H(u)\|c(u))\geq 0$, the minimum is attained by choosing
\begin{equation}
c^\star(u)=\pi_H(u)
\end{equation}
for $P_H$-almost every $u$. Therefore,
\begin{equation}
R^\star_{\log,P}(H)
=
\mathbb E[-\log \pi_H(Y)]
=
\Ent(Y\mid H).
\end{equation}

Taking $H=X$ gives
\begin{equation}
R^\star_{\log,P}
=
R^\star_{\log,P}(X)
=
\mathbb E[-\log \pi_X(Y)]
=
\Ent(Y\mid X).
\end{equation}
Thus
\begin{equation}
\mathcal E_{\log,P}(H)
=
\Ent(Y\mid H)-\Ent(Y\mid X).
\end{equation}

It remains to identify the entropy difference with the expected conditional KL divergence. The difference satisfies
\begin{equation}
\Ent(Y\mid H)-\Ent(Y\mid X)
=
\mathbb E[-\log \pi_H(Y)+\log \pi_X(Y)].
\end{equation}
Conditioning on $X$, and using that $H=h(X)$, the random variable $\pi_H$ is $\sigma(X)$-measurable. Hence
\begin{equation}
\mathbb E[-\log \pi_H(Y)+\log \pi_X(Y)\mid X]
=
\sum_{y\in\mathcal Y}
\pi_X(y)
\log\frac{\pi_X(y)}{\pi_H(y)}.
\end{equation}
The right-hand side is $D_{\mathrm{KL}}(\pi_X\|\pi_H)$. Taking expectations yields
\begin{equation}
\Ent(Y\mid H)-\Ent(Y\mid X)
=
\mathbb E\left[
D_{\mathrm{KL}}(\pi_X\|\pi_H)
\right].
\end{equation}

Finally, since $H=h(X)$, $\sigma(H)\subseteq\sigma(X)$. For finite $\mathcal Y$,
\begin{equation}
I(Y;X\mid H)
=
\Ent(Y\mid H)-\Ent(Y\mid X).
\end{equation}
Therefore,
\begin{equation}
\mathcal E_{\log,P}(H)
=
I(Y;X\mid H).
\end{equation}
\end{proof}

\subsection{Proof of Proposition~\ref{prop:binary-accuracy-calibration-gap}}
\label{app:proof-binary-accuracy-calibration-gap}

\begin{proof}
Binary log loss elicits the conditional probability
\begin{equation}
\eta(X)=P(Y=1\mid X).
\end{equation}
Therefore, a representation $H$ is sufficient for binary log loss if and only if $\eta(X)$ is measurable with respect to $\sigma(H)$. By assumption, $\eta(X)$ is not measurable with respect to $\sigma(H^\star_{0/1})$. Hence $H^\star_{0/1}$ is not sufficient for binary log loss.

Next,
\begin{equation}
P(Y=1\mid H^\star_{0/1})
=
\mathbb E[Y\mid H^\star_{0/1}].
\end{equation}
Since $\eta(X)=\mathbb E[Y\mid X]$ and $H^\star_{0/1}$ is a measurable function of $X$, the tower property gives
\begin{equation}
\mathbb E[Y\mid H^\star_{0/1}]
=
\mathbb E[\mathbb E[Y\mid X]\mid H^\star_{0/1}]
=
\mathbb E[\eta(X)\mid H^\star_{0/1}].
\end{equation}
Therefore,
\begin{equation}
P(Y=1\mid H^\star_{0/1})
=
\mathbb E[\eta(X)\mid H^\star_{0/1}].
\end{equation}

Applying Theorem~\ref{thm:exact-log-loss-gap} in the binary case yields
\begin{equation}
\mathcal E_{\log,P}(H^\star_{0/1})
=
\mathbb E\left[
D_{\mathrm{KL}}
\left(
\operatorname{Bernoulli}(\eta(X))
\,\middle\|\,
\operatorname{Bernoulli}(P(Y=1\mid H^\star_{0/1}))
\right)
\right].
\end{equation}
Substituting the preceding identity gives
\begin{equation}
\mathcal E_{\log,P}(H^\star_{0/1})
=
\mathbb E\left[
\mathrm{kl}
\left(
\eta(X)
\,\middle\|\,
\mathbb E[\eta(X)\mid H^\star_{0/1}]
\right)
\right].
\end{equation}

Finally, since $\eta(X)$ is not $\sigma(H^\star_{0/1})$-measurable, the defining property of conditional expectation implies
\begin{equation}
P\left(
\eta(X)
\neq
\mathbb E[\eta(X)\mid H^\star_{0/1}]
\right)>0.
\end{equation}
The binary KL divergence satisfies $\mathrm{kl}(u\|v)\geq 0$, with equality if and only if $u=v$ under the usual support conventions. Therefore the integrand is strictly positive on a set of positive probability, and hence
\begin{equation}
\mathcal E_{\log,P}(H^\star_{0/1})>0.
\end{equation}
\end{proof}

\end{document}